\documentclass[twoside,11pt]{article}

\usepackage[nohyperref]{jmlr2e}
\usepackage{amssymb}
\usepackage{natbib}
\usepackage{color}
\usepackage{dsfont}
\newcommand{\mathbbm}[1]{\mathds{#1}}
\usepackage[hidelinks,breaklinks=true]{hyperref}
\definecolor{localrevisioncolor}{rgb}{0,0,0}
\newcommand{\localrevision}[1]{{\color{localrevisioncolor}\hypersetup{citecolor=localrevisioncolor,linkcolor=localrevisioncolor}#1}}
\usepackage{amsfonts}       
\usepackage{url}
\usepackage{amsmath}
\usepackage{amssymb}
\usepackage{wrapfig}
\usepackage{mathtools}
\usepackage{tikz}
\usepackage{algorithmic}
\usepackage{footmisc}
\usepackage{bibentry}
\usepackage{thmtools, thm-restate}
\nobibliography*
\usepackage{mathtools}

\usepackage{physics}
\usepackage{mathrsfs}
\mathtoolsset{showonlyrefs}
\usepackage{etoolbox}
\usepackage{makecell}
\usepackage[ruled]{algorithm2e}
\DontPrintSemicolon
\usepackage{enumerate}
\usepackage{booktabs}
\usepackage{longtable}
\usepackage{pifont}
\usepackage{soul}
\usepackage[normalem]{ulem}
\usepackage{cancel}
\usepackage{microtype}
\usepackage{needspace}

\newcommand{\explain}[1]{\tag*{(#1)}}

\newcommand{\cmark}{\ding{51}}%

\newcommand{\fS}{\mathcal{S}}
\newcommand{\fA}{\mathcal{A}}
\newcommand{\fY}{\mathcal{Y}}

\newcommand{\fF}{\mathcal{F}}
\newcommand{\fO}{\mathcal{O}}

\newcommand{\R}{\mathbb{R}}

\newcommand{\E}{\mathbb{E}}

\newcommand{\ns}{{|\fS|}}

\newcommand{\bop}{\mathcal{T}}
\newcommand{\tb}[1]{{\textbf{#1}}}

\newcommand{\indot}[2]{{\left<#1, #2\right>}}

\newcommand{\tref}[1]{\text{\ref{#1}}}
\newcommand{\aref}[1]{\text{A\ref{#1}}}
\newcommand{\e}[1]{\norm{w_{t_{#1}} - w_*}_M^2}

\allowdisplaybreaks

\newcounter{assucounter}
\numberwithin{assucounter}{section}
\newtheorem{assumption}[assucounter]{Assumption}
\newtheorem{assumptionalt}{Assumption}[assumption]
\newenvironment{assumptionp}[1]{
  \renewcommand\theassumptionalt{#1}
  \assumptionalt
}{\endassumptionalt}

\usepackage{lastpage}
\jmlrheading{27}{2026}{1-\pageref{LastPage}}{12/24; Revised 9/26}{11/26}{24-2075}{Xiaochi Qian, Zixuan Xie, Xinyu Liu, Shangtong Zhang}
\ShortHeadings{Almost Sure Convergence Rates and Concentration of Stochastic Approximation}{Qian, Xie, Liu, and Zhang}

\firstpageno{1}

\begin{document}

\title{Almost Sure Convergence Rates and Concentration of Stochastic Approximation and Reinforcement Learning \\with Markovian Noise}

\author{\name Xiaochi Qian$^*$ \email xiaochi.joe.qian@gmail.com \\
\addr Department of Computer Science\\
\addr University of Oxford \\
\addr Wolfson Building, Parks Rd, Oxford, OX1 3QD, UK
\AND
\name Zixuan Xie$^*$ \email xie.zixuan@email.virginia.edu\\
  \addr Department of Computer Science\\
  \addr University of Virginia\\
  85 Engineer's Way, Charlottesville, VA, 22903, United States
\AND
\name Xinyu Liu$^*$ \email xinyuliu@virginia.edu\\
  \addr Department of Computer Science\\
  \addr University of Virginia\\
  85 Engineer's Way, Charlottesville, VA, 22903, United States
\AND
\name Shangtong Zhang \email shangtong@virginia.edu \\
  \addr Department of Computer Science\\
  \addr University of Virginia\\
  85 Engineer's Way, Charlottesville, VA, 22903, United States
  }

\editor{Zhaoran Wang}

\maketitle

\begin{abstract}
  This paper establishes, to our knowledge, the first almost sure convergence rate and the first maximal concentration bound with sub-Weibull tails for general contractive stochastic approximation algorithms with Markovian noise. As a corollary, we also obtain convergence rates in $L^p$. Key to our successes is a novel discretization of the mean ODE of stochastic approximation algorithms using intervals with diminishing (instead of constant) length. Applications yield the first almost sure convergence rate for $Q$-learning without count-based learning rates and the first maximal concentration bound for off-policy temporal difference learning, both with Markovian samples.
  \let\svthefootnote\thefootnote
  \let\thefootnote\relax\footnotetext{$^*$ indicates equal contribution.}
  \let\thefootnote\svthefootnote
\end{abstract}


\section{Introduction}
Stochastic approximation refers to a class of algorithms that update some weight vector iteratively and stochastically \citep{robbins1951stochastic,benveniste1990MP,kushner2003stochastic,borkar2009stochastic}.
Notable stochastic approximation algorithms include stochastic gradient descent (SGD, \citealp{kiefer1952Stochastic}), $Q$-learning \citep{watkins1989learning,watkins1992q}, and temporal difference learning (TD, \citealp{sutton1988learning}).
Typically,
a stochastic approximation algorithm adopts the following form
\begin{align}
  \label{eq sa update}
  \tag{SA}
  w_{t+1} =& w_t + \alpha_t (H(w_t, Y_{t+1}) - w_t).
\end{align}
Here $\qty{w_t \in \R^d}$ are the stochastic iterates, $\qty{\alpha_t}$ is a sequence of learning rates,
$\qty{Y_t}$ is a sequence of random noise evolving in some space $\fY$,
and $H: \R^d \times \fY \to \R^d$ is the function that generates the incremental update.
Throughout, $w_0\in\R^d$ is fixed and deterministic.
The algorithm~\eqref{eq sa update} can be viewed as a generalization of the running average.
Indeed, if $H(w, y)$ is independent of $w$, the learning rate $\alpha_t$ has the form of $\alpha_t = 1/(t+1)$, and $\qty{Y_t}$ are i.i.d.,
then $w_t$ can be regarded as the running average for computing $\E\qty[H(Y_t)]$.
The law of large numbers (LLN) guarantees the almost sure convergence of a running average to the true expectation.
Similarly,
numerous works have been able to prove the almost sure convergence of $\qty{w_t}$ in~\eqref{eq sa update} \citep{benveniste1990MP,kushner2003stochastic,borkar2009stochastic,borkar2025ode,liu2025ode}.
It is well known that LLN is accompanied by a few other theorems that characterize the fluctuation of the running average,
e.g., 
the central limit theorem (CLT), 
the functional central limit theorem (FCLT),
Hoeffding's inequality,
and the law of the iterated logarithm (LIL).
Similarly,
versions of CLT and FCLT have also been well established for the stochastic iterates $\qty{w_t}$ in~\eqref{eq sa update} \citep{benveniste1990MP,borkar2009stochastic,borkar2025ode}.
But the development of a Hoeffding's inequality (cf. a maximal concentration bound with exponential tails) and a LIL (cf. an almost sure convergence rate) for~\eqref{eq sa update} is less satisfactory,
as summarized by
Tables~\ref{tbl sa works} and~\ref{tbl sa works concentration}.
We defer a detailed discussion of those works to Section~\ref{sec related}.
But in short,
previous works usually fall short in two aspects.
The first is that many previous works assume $H$ is linear and then expand their analysis based on products of Hurwitz matrices.
But in many important reinforcement learning (RL, \citealp{sutton2018reinforcement}) algorithms,
the corresponding $H$ is nonlinear.
The second is that many previous works assume $\qty{Y_t}$ are i.i.d.
But in many important RL algorithms,
$\qty{Y_t}$ is a Markov chain.

\begin{table}[b]
  \centering
\begin{tabular}{c|c|c}
   \hline & $H$ & Markovian $\qty{Y_t}$ \\\hline
   \citet{szepesvari1997asymptotic} & $Q$-learning & \\ \hline
   \citet{tadic2002almost} & Linear TD & \cmark \\ \hline
   \citet{chong1999noise} & Linear & \cmark \\ \hline
   \citet{tadic2004almost} & Linear & \cmark \\ \hline
   \citet{kouritzin2015convergence} & Linear & \\ \hline
   \citet{koval2003law} & General & \\ \hline
   \citet{karandikar2024convergence} & General & \\ \hline
   Theorem~\ref{thm lil} & General & \cmark \\ \hline
\end{tabular}
\caption{\label{tbl sa works} Almost sure convergence rates of stochastic approximation algorithms. 
Details are in Section~\ref{sec related}.
}
\end{table}

\begin{table}[t]
  \centering
\begin{tabular}{c|c|c|c|c|c|c}
   \hline & $H$ & $\qty{Y_t}$ & $\alpha_t$ & $\forall t$ & $\ln(\frac{1}{\delta})$ & \cancel{$\sup \norm{w_t}$} \\\hline
   \citet{even2003learning} & $Q$-learning & & \cmark & & \cmark &\\ \hline
   \citet{li2021tightening} & $Q$-learning & & & & \cmark & \\ \hline
   \citet{li2024q} & $Q$-learning & \cmark & & & \cmark & \\ \hline
   \citet{korda2015td} & Linear TD & \cmark & \cmark & \cmark & \cmark & \cmark \\ \hline
   \citet{dalal2018finite} & Linear TD & & \cmark & & \cmark & \cmark \\ \hline
   \citet{chandak2026concentration} & Linear TD & \cmark & \cmark & \cmark & & \cmark \\ \hline
   \citet{prashanth2021concentration} & Batch TD & & \cmark & \cmark & \cmark & \\ \hline
   \citet{dalal2018finitesample} & Linear & & & & \cmark & \cmark  \\ \hline
   \citet{durmus2021tight} & Linear & & \cmark & \cmark & \cmark & \cmark \\ \hline
   \citet{qu2020finite} & General & \cmark & \cmark & \cmark & \cmark&  \\ \hline
   \citet{thoppe2019concentration} & General & & \cmark & & \cmark \\ \hline
   \citet{chandak2022concentration} & General & \cmark & \cmark & \cmark & \cmark \\ \hline
   \citet{mou2022optimal} & General & & & & \cmark & \cmark\\ \hline
   \citet{chen2025concentration} & General & & \cmark & \cmark & \cmark & \cmark\\ \hline
   Theorem~\ref{thm concentration} & General & \cmark & \cmark & \cmark & \cmark & \cmark\\ \hline
   \citet{agrawal2026concentration} & General & \cmark & \cmark & \cmark & \cmark & \cmark\\ \hline
\end{tabular}
\caption{\label{tbl sa works concentration} 
Concentration of stochastic approximation algorithms. 
``$\qty{Y_t}$'' is checked if Markovian noise is allowed. 
``$\alpha_t$'' is checked if there is no need to tune the learning rate during the execution of the algorithm based on runtime statistics. 
``$\forall t$'' stands for ``maximal'' concentration bound and is checked if the bound holds for all $t$ or it holds for $\forall t \geq t_0$ with $t_0$ being some deterministic constant, independent of sample path and the probability parameter $\delta$.
``$\ln\qty(\frac{1}{\delta})$'' is checked if the error bound depends polynomially on $\log(1/\delta)$, e.g., sub-Weibull tails.
``\cancel{$\sup \norm{w_t}$}'' is checked if establishing the concentration bound does not require a priori that the iterates are bounded by some deterministic constant almost surely.
Common ways for obtaining such an a priori bound in previous works are (1) to assume it directly or make some other stronger assumptions, (2) to use specific properties of the specific algorithms (e.g., $Q$-learning), (3) to use a projection operator.
Details are in Section~\ref{sec related}.
Theorem~\ref{thm concentration} is the first to combine these properties under a log-damped harmonic learning rate and uniform geometric mixing.
\citet{agrawal2026concentration} is a subsequent work that allows for a broader choice of learning rates than our Theorem 4.
} 
\end{table}

\tb{Contributions.} This work makes two key contributions towards
closing the aforementioned gaps.
Namely,
this work allows $\qty{Y_t}$ to be a Markov chain
and 
$H(w, y)$ to be nonlinear,
as long as $H(w, y)$ is a noisy estimation of some contractive operator.
Under these weak assumptions,
this work establishes, to our knowledge, two previously unavailable guarantees for the iterates $\qty{w_t}$ generated by~\eqref{eq sa update}: \tb{(1)} an almost sure convergence rate (Table~\ref{tbl sa works}) and \tb{(2)} a maximal concentration bound with sub-Weibull tails (Table~\ref{tbl sa works concentration}).
By ``maximal'',
we mean that our bound holds for all $t$. As a corollary,
we also obtain $L^p$ convergence rates.
Notably,
as it can be seen from Tables~\ref{tbl sa works} and~\ref{tbl sa works concentration},
this work is the only work that appears simultaneously in both tables.
We also note that in those tables, we do not include works that only consider SGD as many RL algorithms are not SGD \citep{sutton2018reinforcement}.
The discussion of SGD is deferred to Section~\ref{sec related}.
As applications (Section~\ref{sec rl app}),
we obtain, to our knowledge, the first almost sure convergence rate for $Q$-learning with Markovian samples without count-based learning rates and the first maximal concentration bound for off-policy TD with Markovian samples.

\tb{Technical Innovation.}
A common approach to studying the stochastic and discrete iterates in~\eqref{eq sa update} is to relate them to deterministic and continuous trajectories of the mean ODE
  $\dv{w(t)}{t} = h(w(t)) - w(t)$ with $h(w)$ denoting the expectation of $H(w, y)$.
This ODE-based approach has seen celebrated success in establishing the desired almost sure convergence in past decades \citep{benveniste1990MP,bertsekas1996neuro,kushner2003stochastic,borkar2009stochastic,borkar2025ode,liu2025ode}.
In those works,
the non-negative real axis is divided into intervals with a fixed length $T$,
corresponding to Euler's discretization of the ODE with an interval $T$ up to some noise.
The key technical innovation in our work is to divide the non-negative real axis into intervals with diminishing length $\qty{T_m}$.
Namely, we will construct a sequence of anchors $\qty{t_m}_{m=0,1,\dots}$ such that the distance between two anchors, $\bar \alpha_m = \sum_{t=t_m}^{t_{m+1}-1} \alpha_t$, is roughly $T_m$. 
Then we examine the iterates $\qty{w_t}$ in~\eqref{eq sa update} in the timescale $\qty{t_m}$.
In other words, we rewrite~\eqref{eq sa update} as~\eqref{eq skeleton sa}
\begin{align}
  \label{eq skeleton sa}
  \tag{Skeleton SA}
  w_{t_{m+1}} = w_{t_m} + \bar \alpha_m \qty(h(w_{t_m}) - w_{t_m}) + z_m,
\end{align}
where $z_m$ is the new noise and $\qty{\bar \alpha_m}$ are the new learning rates.
Since this $z_m$ accumulates the noise between $t_m$ and $t_{m+1}$,
it exhibits some useful averaging effects that facilitate the analysis.
If we follow previous work \citet{bertsekas1996neuro} and use a constant $T$ to discretize the ODE,
then $\qty{\bar \alpha_m}$ is almost constant and
this~\eqref{eq skeleton sa} has a roughly constant learning rate.
It is well known that a constant learning rate stochastic approximation algorithm, in general, does not converge almost surely \citep{kushner2003stochastic}.
By using a diminishing $\qty{T_m}$ to discretize the ODE,
the new learning rates $\qty{\bar \alpha_m}$ also diminish, and 
we can then expect~\eqref{eq skeleton sa} to converge almost surely and further obtain an almost sure convergence rate and a concentration bound based on supermartingales.
The essence of this approach is later on formally verified by \citet{zhang2026rllean} in Lean.

\tb{Notations.} We use $\indot{\cdot}{\cdot}$ to denote the standard Euclidean inner product.
We use $\fO(\cdot)$ and $\Theta(\cdot)$ to hide deterministic constants for simplifying presentation.
In other words, for two non-negative sequences $\qty{a_n}$ and $\qty{b_n}$,
we say $a_n = \fO(b_n)$ if $\exists n_0, C, \forall n \geq n_0, a_n \leq C b_n$.
We say $a_n = \Theta(b_n)$ if $a_n = \fO(b_n)$ and $b_n = \fO(a_n)$.

\section{Main Results}
\begin{assumption}
  \label{assu markov chain}
  The Markov chain $\qty{Y_t}$ adopts a unique stationary distribution $d_\fY$ and mixes geometrically.
  In other words,
  let $P$ be the transition kernel of $\qty{Y_t}$ and $P^n$ be the $n$-step transition kernel.
  Then there exist some constants $\varrho \in [0, 1)$ and $C_\aref{assu markov chain}$ such that for any $y, n$
  \begin{align}
    \textstyle \int_{\fY} \abs{P^n(y, y') - d_\fY(y')} \dd y' \leq C_\aref{assu markov chain} \varrho^n.
  \end{align}
\end{assumption}
This geometrically mixing assumption is standard in analyses concerning Markovian noise, see, e.g., \citet{bertsekas1996neuro,srikant2019finite,borkar2025ode}.
Define $h(w) \doteq \E_{y\sim d_\fY}\qty[H(w, y)]$.
Then, the expected update $h(w) - w$ can be viewed as the residual of the fixed point problem of $h(w) = w$.
This motivates the following contraction assumption.
\begin{assumption}
  \label{assu contraction}
  There exist some $\kappa \in [0, 1)$ and some norm $\norm{\cdot}$ such that for any $w, w'$, 
  \begin{align}
    \norm{h(w) - h(w')} \leq \kappa \norm{w - w'}.
  \end{align}
\end{assumption}
We then use $w_*$ to denote the unique fixed point of $h$,
thanks to Banach's fixed point theorem.
Since the norm $\norm{\cdot}$ in Assumption~\ref{assu contraction} is arbitrary,
this contraction assumption is easily satisfied in many RL algorithms.
\begin{remark}[Pseudo-Contraction]
  Similar to \citet{zhang2023convergence,liu2025linearq}, it is easy to verify that all our proofs will still go through if $h$ is merely a pseudo-contraction,
  as long as the existence and uniqueness of a fixed point $w_*$ can be established.
  Example algorithms associated with pseudo-contractions include SARSA with linear function approximation \citep{de2000existence} and tabular $Q$-learning with a changing behavior policy \citep{liu2025linearq}.
\end{remark}
\begin{remark}[Linear Stochastic Approximation]
  \label{rem linear}
For the special case where $h(w) - w$ is linear, i.e., $h(w) - w = Aw + b$,
it commonly holds in RL algorithms that the matrix $A$ is Hurwitz \citep{liu2025ode}.
Then we can rewrite $h(w) - w$ as $\frac{1}{\beta} (\beta Aw + \beta b)$.
The outer $\frac{1}{\beta}$ can be absorbed into the learning rate.
So it is equivalent to just considering $h(w) - w = \beta Aw + \beta b + w - w$.
\citet[Section~6.1]{chen2025concentration} show that $w \mapsto \beta Aw + \beta b + w$ is a contraction w.r.t. some weighted $\ell_2$ norm for some carefully chosen $\beta$.
This means our results apply to linear stochastic approximation \citep{lakshminarayanan2018linear,srikant2019finite} in the form of
\begin{align}
  \label{eq linear sa}
  \tag{Linear SA}
  w_{t+1} = w_t + \alpha_t (A(Y_{t+1})w_t + b(Y_{t+1})) 
\end{align}
as well,
provided that $\E_{Y\sim d_\fY}\qty[A(Y)]$ is Hurwitz,
Assumption~\ref{assu markov chain} holds, and the rescaled update satisfies
Assumption~\ref{assu Lipschitz}.
\end{remark}
\begin{assumption}
  \label{assu Lipschitz}
  There exists some finite $L_h$ such that for any $w, w', y$, it holds that
  \begin{align}
    \label{eq G Lipschitz}
    \norm{H(w, y) - H(w', y)} \leq L_h \norm{w - w'}.
  \end{align}
  Moreover, $\sup_{y \in \fY} \norm{H(0, y)} < \infty$.
\end{assumption}


\begin{assumptionp}{LR1}
  \label{assu lr}
  The learning rate has the form 
  \begin{align}
    \textstyle\alpha_t = \frac{C_\alpha}{(t+3)^{\nu}} \qq{with} \nu \in (\frac{2}{3}, 1].
  \end{align}
\end{assumptionp}
\begin{assumptionp}{LR2}
  \label{assu lr2}
  The learning rate has the form 
  \begin{align}
    \textstyle\alpha_t = \frac{C_\alpha}{(t+3)\ln^\omega(t+3)} \qq{with} \omega \in (0, 1).
  \end{align}
\end{assumptionp}
Here we take $C_\alpha>0$.
The key challenge in designing the~\eqref{eq skeleton sa} is to carefully balance the choice of $\alpha_t$ and $T_m$ such that the new noise $z_m$ has the proper averaging effect. 
The lower bound $\frac{2}{3}$ and the $\ln^\omega(t+3)$ term are all artifacts from this balance. 
Assumption~\ref{assu lr} will be used for establishing the almost sure convergence rates.
Assumption~\ref{assu lr2} will be used for establishing the concentration bounds.
The learning-rate parameter $\nu$ controls polynomial decay under Assumption~\ref{assu lr}, whereas $\omega$ controls the logarithmic correction under Assumption~\ref{assu lr2}. A smaller fixed $\omega>0$ gives a faster asymptotic rate in our concentration bound. However, the constants in this bound may depend on $\omega$, and we do not establish that they remain bounded as $\omega\downarrow0$.

\begin{theorem}[Almost Sure Convergence Rates]
  \label{thm lil}
  Let Assumptions \ref{assu markov chain} -~\ref{assu Lipschitz}  and~\ref{assu lr} hold.
  Let $\qty{w_t}$ be the iterates generated by~\eqref{eq sa update}.
If $\nu < 1$, then for any $\zeta \in (0, \frac{3}{2} \nu - 1)$, 
  \begin{align}
    \textstyle\lim_{t\to\infty} t^\zeta \norm{w_t - w_*}^2 = 0 \qq{a.s.}
  \end{align}
If $\nu = 1$, then for any $\zeta \in (0, 1)$ and $\beta_1 > 0$,
  \begin{align}
    \textstyle\lim_{t\to\infty} \exp(\zeta \ln^{1/(1+\beta_1)}t) \norm{w_t - w_*}^2 = 0 \qq{a.s.}
  \end{align}
\end{theorem}
The proof is in Section~\ref{sec main proof}.
To our knowledge,
Theorem~\ref{thm lil} is the first almost sure convergence rate for general contractive stochastic approximation algorithms with Markovian noise (Table~\ref{tbl sa works}).

\begin{theorem}[Concentration]
  \label{thm concentration}
  Let Assumptions~\ref{assu markov chain} -~\ref{assu Lipschitz}  and~\ref{assu lr2} hold.
  Then for any $\omega \in (0, 1)$,
  there exist some deterministic constants $\bar C_\alpha, C_{Th\tref{thm concentration}}, C_{Th\tref{thm concentration}}',$ and integer $C_{\tref{lem concentration in m}}$ such that
  the iterates $\qty{w_t}$ generated by~\eqref{eq sa update} satisfy
  the following concentration property whenever $C_\alpha \geq \bar C_\alpha$:
  for any $\delta \in (0, 1)$,
  it holds,
  with probability at least $1 - \delta$,
  that for all $t \geq 0$,
  \begin{align}
    \textstyle \norm{w_t - w_*}^2 \leq C_{\text{Th}\tref{thm concentration}}\exp(\frac{-\ln^{1-\omega} (t+1)}{1-\omega})\left[\log(\frac{1}{\delta}) + C_{\text{Th}\tref{thm concentration}}' + \frac{\ln^{1-\omega} (t+1)}{1-\omega}\right]^{C_{\tref{lem concentration in m}}}.
  \end{align}
\end{theorem}
The proof is in Section~\ref{sec main proof}.
To our knowledge,
Theorem~\ref{thm concentration} is the first maximal concentration bound with sub-Weibull tails for general contractive stochastic approximation algorithms with Markovian noise (Table~\ref{tbl sa works concentration}).
Define $s(t)\doteq\frac{\ln^{1-\omega}(t+1)}{1-\omega}$,
the maximal bound then yields the following two corollaries: an almost sure rate under \ref{assu lr2} and $L^{2p}$ convergence rates.

\begin{corollary}[Almost Sure Rate under \ref{assu lr2}]
\label{cor lr2 as}
Let the conditions of Theorem~\ref{thm concentration} hold. Then
\begin{align}
\textstyle\norm{w_t-w_*}^2=O\bigl(e^{-s(t)}(1+s(t))^{C_{\tref{lem concentration in m}}}\bigr)\quad\text{a.s.},
\end{align}
where the implicit constant may depend on the sample path. In particular, for every $c\in(0,1)$,
\begin{align}
\textstyle\lim_{t\to\infty}e^{c s(t)}\norm{w_t-w_*}^2=0\quad\text{a.s.}
\end{align}
\end{corollary}
The proof is in Section~\ref{sec proof cor lr2 as}. This is an almost sure rate under \ref{assu lr2}, complementary to Theorem~\ref{thm lil} under \ref{assu lr}.

\begin{corollary}[$L^{2p}$ Convergence Rates]
  \label{cor lp}
  Let the conditions of Theorem~\ref{thm concentration} hold and set
  $a(t)=C_{\mathrm{Th}\tref{thm concentration}}e^{-s(t)}$, and
  $b(t)=C'_{\mathrm{Th}\tref{thm concentration}}+s(t)$, enlarging the latter constant to be nonnegative if necessary.
  Here $C_{\tref{lem concentration in m}}\geq1$ is the integer in Theorem~\ref{thm concentration}.
  For every real $p\geq1$ and $t\geq0$,
  \begin{align}
    \E\qty[\norm{w_t-w_*}^{2p}]
    \leq 2^{C_{\tref{lem concentration in m}}p-1}a(t)^p\bigl[b(t)^{C_{\tref{lem concentration in m}}p}+\Gamma(C_{\tref{lem concentration in m}}p+1)\bigr].
  \end{align}
  In particular, for each fixed $p\geq1$, the RHS is 
  $O(e^{-p s(t)}(1+s(t))^{C_{\tref{lem concentration in m}}p})$.
\end{corollary}
The proof is in Section~\ref{sec proof cor lp}. Taking $p=1$ gives a mean-square convergence rate under \ref{assu lr2}.

\tb{Rates and learning-rate regimes.}
The two theorems concern different learning-rate families and neither directly implies the other as stated. Under \ref{assu lr} with $\nu<1$, Theorem~\ref{thm lil} gives $\norm{w_t-w_*}^2=o(t^{-\zeta})$ almost surely for $\zeta<3\nu/2-1$. At $\nu=1$, its guaranteed rate is sub-polynomial: for every fixed $\beta_1>0$ and $\zeta>0$, $\exp(\zeta\ln^{1/(1+\beta_1)}t)=o(t^r)$ for every $r>0$. Taking $\beta_1$ arbitrarily small does not yield a polynomial rate for any fixed positive $\beta_1$.
Under \ref{assu lr2}, Theorem~\ref{thm concentration} gives a sub-polynomial decay factor $e^{-s(t)}$, multiplied by a polynomial in $s(t)$ and $\log(1/\delta)$. The boundary $\omega=0$ is not covered by the theorem.
For \ref{assu lr}, we do not claim optimality of the polynomial exponent at $\nu<1$ or tightness of the sub-polynomial bound at $\nu=1$. Nor do we claim optimality of the precise constants or logarithmic factors in either theorem. However, the sub-polynomial rate under \ref{assu lr2} is unavoidable in general as Proposition~\ref{prop lr2 lower} below shows. In particular, the nearly inverse-linear rate in skeleton time does not imply a nearly inverse-linear rate in the original iteration index. Mean-square analyses, such as \citet{chen2024lyapunov}, can give sharper time dependence under their respective stepsize conditions, but an expectation bound alone does not give the same almost sure rate: Markov's inequality followed by Borel--Cantelli requires summable failure probabilities. Our results provide pathwise and maximal-in-time guarantees.

\begin{proposition}[Sub-polynomial Decay Is Unavoidable under \ref{assu lr2}]
\label{prop lr2 lower}
For every $\omega\in(0,1)$, there is a deterministic scalar instance satisfying the conditions of Theorem~\ref{thm concentration}, with $C_\alpha=5$, for which
\begin{align}
\textstyle\abs{w_t-w_*}^2=\Theta\bigl(\exp\{-5\ln^{1-\omega}(t+3)/(1-\omega)\}\bigr).
\end{align}
Consequently, $t^r\abs{w_t-w_*}^2\to\infty$ for every $r>0$.
\end{proposition}
The proof is in Section~\ref{sec proof prop lr2 lower}.

No $O(t^{-r})$ squared-error guarantee with $r>0$ can therefore hold
for all instances covered by our theorem under \ref{assu lr2}, even without noise.
This also rules out such a general mean-square or fixed-confidence guarantee,
since the example is deterministic. Thus our bound has the correct stretched-logarithmic time scale
$\exp\{-\Theta((\log t)^{1-\omega})\}$ for \ref{assu lr2}: its sub-polynomial character is intrinsic to the stepsize schedule. 
The extra factor $\ln^\omega(t+3)$ in the denominator makes the stepsizes in \ref{assu lr2} smaller than harmonic stepsizes. This reduction alone can prevent polynomial convergence, as Proposition~\ref{prop lr2 lower} shows. The harmonic-stepsize result of \citet{agrawal2026concentration} concerns a different schedule; we discuss its guarantees and assumptions in Section~\ref{sec related}.

\tb{Tail shape and limitations.}
With $a(t)$, $b(t)$, and $C_{\tref{lem concentration in m}}$ as in Corollary~\ref{cor lp}, for fixed $t$, Theorem~\ref{thm concentration} has the form
\begin{align}
\textstyle\Pr\{\norm{w_t-w_*}^2>a(t)[b(t)+u]^{C_{\tref{lem concentration in m}}}\}\leq e^{-u},\qquad u>0,
\end{align}
with $a(t)>0$ and $b(t)\geq0$. Thus its tail in the squared-error threshold $x$ is bounded by $\exp\{b(t)-(x/a(t))^{1/C_{\tref{lem concentration in m}}}\}$ for $x\geq a(t)b(t)^{C_{\tref{lem concentration in m}}}$. This is a sub-Weibull bound of shape $1/C_{\tref{lem concentration in m}}$ when $C_{\tref{lem concentration in m}}>1$. The exponent $C_{\tref{lem concentration in m}}$ comes from the number of bootstrap steps; we do not establish its optimality for Markovian noise. The i.i.d. analysis of \citet{chen2025concentration} obtains stronger tails for additive noise than for multiplicative noise. Obtaining a corresponding improvement for additive Markovian noise is left for future work.

\tb{Stepsize scale and parameter dependence.}
The sufficient lower threshold on $C_\alpha$ grows as $1/(1-\kappa)$
so weaker contraction requires a larger stepsize coefficient in our concentration analysis. Neither learning-rate regime requires $\alpha_t\leq1$ at every iteration; a large $C_\alpha$ can cause a large initial transient.
Slower mixing increases the Markovian-bias bound through the factor
$C_\aref{assu markov chain}\varrho/(1-\varrho)$ in the proof of Lemma~\ref{lem bound z3}.
The Lipschitz and update scales enter through
\begin{align}
\textstyle\norm{H(w,y)-w}\leq(L_h+1)\norm{w}+\sup_y\norm{H(0,y)},
\qquad
\norm{w_*}\leq\frac{\sup_y\norm{H(0,y)}}{1-\kappa}.
\end{align}
The second inequality follows from $h(w_*)=w_*$ and contraction. Larger values of these scales increase the noise bounds and the Gronwall bound on the initial transient (Lemmas~\ref{lem bound G}--\ref{lem bound gronwall}). 

\section{Reinforcement Learning Applications}
\label{sec rl app}
We apply our general results to asynchronous $Q$-learning and off-policy TD under Markovian sampling. The resulting guarantees concern almost sure convergence rates and maximal concentration of the value estimates.

We consider a Markov Decision Process \citep{bellman1957markovian,puterman2014markov} with a finite state space $\fS$,
a finite action space $\fA$,
a reward function $r: \fS \times \fA \to \R$,
a transition function $p: \fS \times \fS \times \fA \to [0, 1]$,
an initial distribution $p_0$,
and a discount factor $\gamma \in [0, 1)$.
At time step 0, an initial state $S_0$ is sampled from $p_0$.
At time step $t$,
an agent at the state $S_t$ takes an action $A_t \sim \mu(\cdot | S_t)$,
where $\mu$ is the behavior policy.
A reward $R_{t+1} \doteq r(S_t, A_t)$ is emitted, and the agent proceeds to a successor state $S_{t+1} \sim p(\cdot | S_t, A_t)$.

\tb{$Q$-learning.}
The goal of $Q$-learning is to find the optimal action value function $q_* \in \R^{\fS\times \fA}$ satisfying
  $q_* = \bop_* q_*$,
where $\bop_*: \R^{\fS \times \fA} \to \R^{\fS \times \fA}$ is the Bellman optimality operator defined as
  $(\bop_* q)(s, a) \doteq r(s, a) + \gamma \sum_{s'} p(s'|s, a) \max_{a'} q(s' ,a')$.
It is well-known that $\bop_*$ is a $\gamma$-contraction w.r.t. $\norm{\cdot}_\infty$ \citep{puterman2014markov}.
As a result, $q_*$ is well-defined and unique.
$Q$-learning is the most powerful and widely used method to find this $q_*$.
In its simplest and most straightforward form,
$Q$-learning performs the following updates to get the stochastic iterates $\qty{q_t}$ as
\begin{align}
  q_{t+1}(s, a) =&\textstyle  q_t(s, a) + \alpha_t (R_{t+1} + \gamma \max_{a'} q_t(S_{t+1}, a') - q_t(s, a)) &\qq{if $(s, a) = (S_t, A_t)$} \\
  q_{t+1}(s, a) =& q_t(s, a) &\qq{otherwise.}
\end{align}
This update rule can be written in a compact form as
\begin{align}
  \label{eq q learning}
  \tag{$Q$-learning}
  q_{t+1} =& q_t + \alpha_t(H(q_t, (S_t, A_t, S_{t+1})) - q_t) \\
  \qty[H(q, (s_0, a_0, s_1))](s, a) \doteq& \textstyle q(s, a) + \mathbbm{1}_{(s, a) = (s_0, a_0)}(r(s_0, a_0) + \gamma \max_{a_1} q(s_1, a_1) - q(s_0, a_0)),
\end{align}
where $\mathbbm{1}$ is the indicator function. The stepsize $\alpha_t$ is indexed by the global iteration $t$, rather than the number of visits to the sampled state-action pair.
Let $d_\mu$ be the stationary state distribution under the behavior policy $\mu$.
Define $h(q) \doteq \E_{s\sim d_\mu, a\sim \mu(\cdot|s), s' \sim p(\cdot | s, a)}\qty[H(q, (s, a, s'))]$.
Proposition 3.1 of \citet{chen2024lyapunov} proves that $h(q)$ is a $\gamma'$-contraction w.r.t. $\norm{\cdot}_\infty$ for some $\gamma' \in [0, 1)$ and has $q_*$ as its unique fixed point,
which immediately allows us to invoke Theorems~\ref{thm lil} and~\ref{thm concentration}.
More explicitly, writing $d_{\min}^{SA}=\min_{s,a}d_\mu(s)\mu(a|s)>0$, the averaged asynchronous operator has contraction modulus $\gamma'=1-(1-\gamma)d_{\min}^{SA}$.
\begin{theorem}
  \label{thm q learning}
  Let the chain $(S_t, A_t, S_{t+1})$ induced by the behavior policy $\mu$ be irreducible and aperiodic.
  If Assumption~\ref{assu lr} holds,
  then the conclusion of Theorem~\ref{thm lil} holds with $\qty{w_t}$ identified as $\qty{q_t}$ generated by~\eqref{eq q learning}.
  If Assumption~\ref{assu lr2} holds and $C_\alpha\geq\bar C_\alpha$ as required in Theorem~\ref{thm concentration},
  then the conclusion of Theorem~\ref{thm concentration} holds with $\qty{w_t}$ identified as $\qty{q_t}$ generated by~\eqref{eq q learning}.
\end{theorem}
\begin{proof}
  Assumption~\ref{assu markov chain} is immediately implied by the finiteness of $\fS\times \fA \times \fS$ and the irreducibility and aperiodicity of the chain $(S_t, A_t, S_{t+1})$.
  The function $H$ is trivially Lipschitz continuous.
  Invoking Theorem~\ref{thm lil} for \ref{assu lr} and Theorem~\ref{thm concentration} for \ref{assu lr2} then completes the proof.
\end{proof}
Theorem~\ref{thm q learning} gives almost sure rates under \ref{assu lr} and maximal concentration under \ref{assu lr2}.
To better understand the contribution of Theorem~\ref{thm q learning},
we briefly discuss three variants of~\eqref{eq q learning}.
The first is \tb{synchronous $Q$-learning},
where $q_t$ is updated for each $(s, a)$, not just $(S_t, A_t)$, at time step $t$. 
This requires drawing a sample $s'$ from $p(\cdot | s, a)$ for each $(s, a)$, and the noise is, therefore, essentially i.i.d.
Moreover,
in the typical RL setup \citep{sutton2018reinforcement}, only one Markovian data stream $(\dots, S_t, A_t, R_{t+1}, S_{t+1}, \dots)$ is available so implementing synchronous $Q$-learning is impractical.
The second variant considers \tb{i.i.d. samples}.
In other words,
at each time step $t$,
a tuple $(S_t, A_t)$ is drawn independently from some distribution and the $S_{t+1}$ in~\eqref{eq q learning} is replaced with $S_t' \sim p(\cdot | S_t, A_t)$.
These i.i.d. samples are again not immediately available in the canonical RL setup.
The third variant is \tb{count-based learning rates}
where the $\alpha_t$ in~\eqref{eq q learning} is replaced by $\alpha_{n(t, S_t, A_t)}$.
Here, $n(t, s, a)$ counts the number of visits to the state action pair $(s, a)$ until time $t$.
To our knowledge,
this count-based learning rate is rarely used by RL practitioners and is mostly an artifact for theoretical analysis.
See~\citet{antrobius2026divergence} for an argument on why count-based learning rates are primarily used for the ease of theoretical analysis.
We, therefore, argue that the form of~\eqref{eq q learning} we consider in this section is the most straightforward and practical one.

Existing analyses of $Q$-learning and its three variants establish
almost sure convergence,
high probability concentration bounds, and
$L^2$ convergence rates \citep{watkins1989learning,watkins1992q,jaakkola1993convergence,tsitsiklis1994asynchronous,kearns1998finite,even2003learning,azar2011speedy,beck2012error,lee2019unified,shah2018q,wainwright2019stochastic,qu2020finite,li2020sample,chen2024lyapunov,chen2025concentration,li2024q}.
In particular, \citet{chen2024lyapunov} analyze asynchronous $Q$-learning with Markovian sampling and provide mean-square bounds.
To our knowledge,
\tb{Theorem~\ref{thm q learning} is the first almost sure convergence rate for $Q$-learning in the form of~\eqref{eq q learning}}.
It turns out that \citet{szepesvari1997asymptotic} is the only almost sure convergence rate of $Q$-learning prior to this work.
It,
however,
uses count-based learning rates and considers i.i.d. samples.\footnote{It is discussed in \citet{szepesvari1997asymptotic} that their methodology would also work for Markovian samples, but there is no formal statement and no formal proof.} 

Policy suboptimality and cumulative regret assess complementary aspects of learning: the quality of the learned policy and the cumulative loss incurred during learning, respectively \citep{menard2021fast}.
Small regret alone need not imply convergence of policy performance, as shown in the episodic setting by \citet{dann2017unifying}.
For $\pi_t$ greedy with respect to $q_t$, the standard policy-loss bound \citep[Corollary~2]{singh1994upper} gives $\norm{v_*-v_{\pi_t}}_\infty\leq\frac{2}{1-\gamma}\norm{q_t-q_*}_\infty$.
Consequently, our almost sure rate controls how fast the extracted policies approach optimal value along almost every sample path. This is relevant when a policy is learned for subsequent deployment. Here the data are collected under a fixed behavior policy $\mu$, so the bound on $\pi_t$ does not control the regret incurred by $\mu$.

\tb{Off-policy TD.}
Instead of finding the optimal action value function $q_*$,
another fundamental task in RL is policy evaluation,
where the goal is to estimate the value function $v_\pi$ for some policy $\pi$, called the target policy.
Here, $v_\pi$ is the unique fixed point of the Bellman operator $\bop_\pi: \R^\ns \to \R^\ns$ defined as $(\bop_\pi v)(s) = \sum_a \pi(a|s) \qty[r(s, a) + \gamma \sum_{s'} p(s'|s, a) v(s')]$.
It is well known that $\bop_\pi$ is also a $\gamma$-contraction w.r.t. $\norm{\cdot}_\infty$ \citep{puterman2014markov}.
When the data is obtained via a behavior policy $\mu$ (i.e., $A_t \sim \mu(\cdot | S_t)$),
we consider the form of off-policy TD that applies importance sampling to the TD target \citep{graves2022importance}. It computes estimates $\qty{v_t}$ of $v_\pi$ as
\begin{align}
  \label{eq td}
  \tag{off-policy TD}
  v_{t+1} =& v_t + \alpha_t(H(v_t, (S_t, A_t, S_{t+1})) - v_t) \\
  \qty[H(v, (s_0, a_0, s_1))](s) \doteq& \textstyle v(s) + \mathbbm{1}_{s = s_0}\left[\frac{\pi(a_0|s_0)}{\mu(a_0|s_0)}(r(s_0, a_0) + \gamma v(s_1)) - v(s_0)\right].
\end{align}
Similarly, define $h(v) \doteq \E_{s\sim d_\mu, a\sim \mu(\cdot|s), s' \sim p(\cdot | s, a)}\qty[H(v, (s, a, s'))]$.
Under the coverage assumption below, importance sampling gives
\begin{align}
\textstyle h(v)=(I-D_\mu)v+D_\mu\bop_\pi v,
\qquad D_\mu=\operatorname{diag}(d_\mu(s):s\in\fS).
\end{align}
Consequently, when $d_\mu(s)>0$ for every state,
\begin{align}
\textstyle\norm{h(v)-h(v')}_\infty
\leq\kappa_{\mathrm{TD}}\norm{v-v'}_\infty,
\qquad \kappa_{\mathrm{TD}}=1-(1-\gamma)\min_s d_\mu(s)<1.
\end{align}
Thus the asynchronous mean operator has a visitation-dependent contraction modulus, while the Bellman operator itself has modulus $\gamma$. Since $D_\mu$ is invertible, $h(v)=v$ is equivalent to $\bop_\pi v=v$, so its unique fixed point is $v_\pi$. Ratios in the sampled update are needed only for actions with $\mu(a|s)>0$.
Let $R_{\max}=\max_{s,a}|r(s,a)|$ and
$\rho_{\max}=\max_{(s,a):\mu(a|s)>0}\pi(a|s)/\mu(a|s)<\infty$.
For every sample $y$ in the support of the behavior chain,
\begin{align}
\norm{H(v,y)-H(v',y)}_\infty\leq\max\{1,\gamma\rho_{\max}\}\norm{v-v'}_\infty,
\qquad \norm{H(0,y)}_\infty\leq\rho_{\max}R_{\max}.
\end{align}
Together with the mean contraction and finite-state mixing, these bounds verify the assumptions of our general SA results.
We then have
\begin{theorem}
  \label{thm td}
  Let the chain $(S_t, A_t, S_{t+1})$ induced by the behavior policy $\mu$ be irreducible and aperiodic.
  Let the behavior policy $\mu$ cover the target policy $\pi$ in the sense that $\mu(a|s) = 0 \implies \pi(a|s) = 0$.
  If Assumption~\ref{assu lr} holds,
  then the conclusion of Theorem~\ref{thm lil} holds with $\qty{w_t}$ identified as $\qty{v_t}$ generated by~\eqref{eq td}.
  If Assumption~\ref{assu lr2} holds and $C_\alpha\geq\bar C_\alpha$ as required in Theorem~\ref{thm concentration},
  then the conclusion of Theorem~\ref{thm concentration} holds with $\qty{w_t}$ identified as $\qty{v_t}$ generated by~\eqref{eq td}.
\end{theorem}
The proof is the same as that of Theorem~\ref{thm q learning} and is thus omitted.
To our knowledge,
\tb{Theorem~\ref{thm td} is the first concentration bound for~\eqref{eq td} with Markovian samples}.
\citet[Section~6.2.1]{chen2025concentration} establish concentration bounds for off-policy TD variants with generalized importance sampling under independent trajectory sampling.
We also note that~\eqref{eq td} adopts many variants, including~\citet{precup2000eligibility,munos2016safe,harutyunyan2016q,de2018multi}.
Our general results also apply to such variants when their mean contraction, sampling, and uniform Lipschitz assumptions are verified; we omit these extensions to simplify the presentation.
We refer the reader to \citet{chen2025concentration} for details of the contraction property of those variants.

\begin{remark}
\label{rem soft q vs q}
A uniform deterministic bound on the iterates is available for~\eqref{eq q learning},
but need not exist for~\eqref{eq td}.
For~\eqref{eq q learning}, let $R_{\max}=\max_{s,a}|r(s,a)|$ and $B_t=\norm{q_t}_\infty$. Whenever $0\leq\alpha_t\leq1$, only the sampled coordinate changes, so
\begin{align}
B_{t+1}\leq\max\{B_t,(1-\alpha_t)B_t+\alpha_t(R_{\max}+\gamma B_t)\}
\leq\max\{B_t,R_{\max}/(1-\gamma)\}.
\end{align}
Both learning-rate families eventually satisfy $\alpha_t\leq1$. For a fixed initialization, the finitely many earlier iterates also have a deterministic bound: the finite reward scale and Lipschitz update give
$B_{t+1}\leq(1+(1+\gamma)\alpha_t)B_t+\alpha_t R_{\max}$.
Hence $\sup_t B_t$ is bounded by a deterministic constant.
This bound is the basis for many analyses of $Q$-learning, e.g., \citet{even2003learning,li2021tightening,li2024q,qu2020finite}.
For~\eqref{eq td}, write $\rho_t=\pi(A_t|S_t)/\mu(A_t|S_t)$.
Whenever $0\leq\alpha_t\leq1$, the sampled coordinate satisfies
\begin{align}
\textstyle |v_{t+1}(S_t)|\leq[1-\alpha_t+\alpha_t\gamma\rho_t]\norm{v_t}_\infty
+\alpha_t\rho_t R_{\max}.
\end{align}
When $\gamma\rho_t>1$, the coefficient of $\norm{v_t}_\infty$ exceeds one,
so the preceding convex-combination argument does not yield a uniform deterministic bound.

Such a bound can in fact fail to exist. Consider one state, two actions,
$\mu(a_1)=\mu(a_2)=1/2$, $\pi(a_1)=1$, $\gamma=0.9$,
$r(a_1)=1$, $r(a_2)=0$, and $v_0=0$.
On the event that the first $n$ actions are all $a_1$, which has probability $2^{-n}$,
each update is $v_{t+1}=(1+0.8\alpha_t)v_t+2\alpha_t$, so
\begin{align}
\textstyle v_n+2.5=2.5\prod_{t=0}^{n-1}(1+0.8\alpha_t).
\end{align}
Under either \ref{assu lr} or \ref{assu lr2}, $\alpha_t\to0$ and $\sum_t\alpha_t=\infty$.
The product tends to infinity because $\log(1+0.8\alpha_t)\geq0.4\alpha_t$
for all sufficiently large $t$.
Thus every finite deterministic bound is exceeded at some finite time on an event of positive probability.
Analyses requiring a uniform deterministic bound on the iterates therefore do not apply directly to this TD update.
This does not contradict almost sure convergence: a convergent sample path is bounded, but its bound can depend on the path.
\end{remark}

\tb{RL with linear function approximation.} In view of Remark~\ref{rem linear}, Theorems~\ref{thm lil} and~\ref{thm concentration} apply when the joint mean update matrix is Hurwitz and the Markov mixing and uniform boundedness/Lipschitz assumptions hold. 
Example algorithms include linear TD \citep{sutton1988learning,tsitsiklis1997analysis},
gradient TD \citep{sutton2008convergent,sutton2009fast,maei2011gradient,zhang2021average},
and density ratio learning methods \citep{nachum2019dualdice,zhang2020gradientdice,zhang2019provably}.
We, however, omit them for simplifying the presentation.

We additionally note that prior to this work,
\citet{szepesvari1997asymptotic,tadic2002almost} are, to our knowledge, the only two works that establish almost sure convergence rates for RL algorithms.
By contrast,
our Theorem~\ref{thm lil} can be applied to a variety of RL algorithms as discussed above.

\section{Numerical Illustration}
\label{sec lr2 experiment}
Numerical results for asynchronous $Q$-learning under Markovian sampling with
the polynomial and rescaled harmonic stepsize types in \ref{assu lr} are already available
in \citet[Section~5, Figure~1]{zhang2024constant}.
We complement these results with a numerical illustration under \ref{assu lr2}.

We use a continuing MDP with three states, two actions, and discount factor
$\gamma=0.5$. Each run follows a single Markovian trajectory under a fixed
uniform behavior policy, using asynchronous $Q$-learning with
$\alpha_t=32/((t+3)\ln^\omega(t+3))$ for $\omega\in\{0.2,0.5,0.8\}$.
The coefficient satisfies the sufficient threshold in Theorem~\ref{thm concentration}.
We run $10^6$ updates for each of 128 independent seeds.
Figure~\ref{fig lr2 experiment} shows the sup-norm error of the unaveraged
iterates for the preselected seed 0 and pointwise empirical quantiles across runs.
The full MDP specification, initialization, stepsize-threshold verification,
and computation of $q_*$ are given in Appendix~\ref{sec experimental details}.

\begin{figure}[htbp]
\centering
\includegraphics[width=\linewidth]{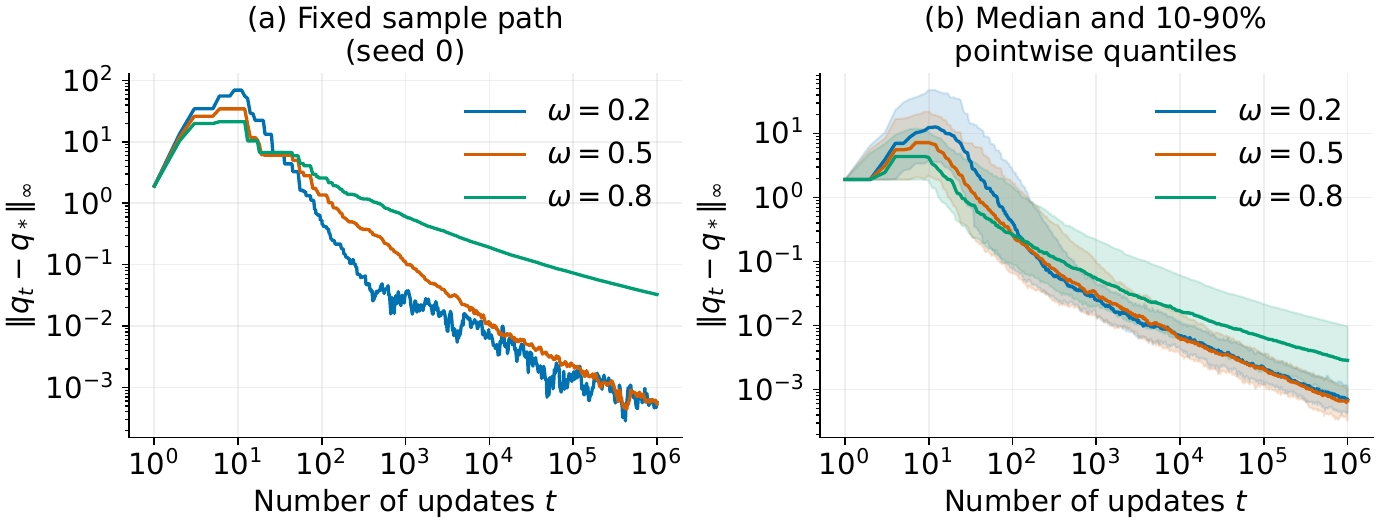}
\caption{\label{fig lr2 experiment}Asynchronous $Q$-learning under \ref{assu lr2}.
Left: a fixed sample path (seed 0). Right: median errors and shaded 10th--90th
percentile intervals over 128 independent runs. The intervals are pointwise
empirical quantiles, not simultaneous confidence bands. All curves retain the
initial large-stepsize transient; no iterate averaging is used.}
\end{figure}

At $t=10^6$, the median sup-norm errors are $6.96\times10^{-4}$,
$6.67\times10^{-4}$, and $2.85\times10^{-3}$ for $\omega=0.2,0.5,0.8$,
respectively. The curves exhibit an initial increase followed by substantial
error reduction. The largest exponent has a slower median decay in this example;
the comparison does not imply a universal finite-time ordering in $\omega$.
This experiment illustrates finite-time behavior under the stated \ref{assu lr2} schedule.

\section{Related Works}
\label{sec related}
We review related work on convergence rates and concentration bounds,
and discuss connections with existing stochastic approximation analyses.

\tb{Almost Sure Convergence Rates.}
Almost sure convergence of SA and RL has been extensively studied in the literature, see, e.g., \citet{benveniste1990MP,tsitsiklis1994asynchronous,tsitsiklis1997analysis,abounadi2001learning,kushner2003stochastic,blaser2026almost,blaser2026asymptotic,wang2024almost}.
But for almost sure convergence rates,
a large body of prior works focus on a special case of~\eqref{eq sa update},
the stochastic gradient descent (SGD),
where the update $H(w_t, Y_{t+1}) - w_t$ is the noisy gradient estimation of some objective and $\qty{Y_t}$ are i.i.d. \citep{pelletier1998almost,godichon2019lp,sebbouh2021almost,liu2022almost,liang2024almost,liu2024almost,weissmann2025almost}.
Despite that SGD has enjoyed celebrated success \citep{lecun2015deep},
in many RL algorithms,
the corresponding update is not a noisy gradient estimation of any objective, and $\qty{Y_t}$ is not i.i.d.
Instead, in RL,
$H(w_t, Y_{t+1})$ is usually a noisy estimation of some contractive operator.
Moreover, in RL, $\qty{Y_t}$ is usually a Markov chain.

Beyond the special SGD case,
the almost sure convergence rate of~\eqref{eq sa update} is previously studied in \citet{chong1999noise,tadic2004almost,kouritzin2015convergence}.
Those works,
however, all assume that $H$ is linear.
For nonlinear $H$, almost sure rates are studied by \citet{koval2003law,karandikar2024convergence}.
\citet{koval2003law} assume i.i.d. Gaussian noise.
\citet{vidyasagar2023convergence} establishes almost sure boundedness and convergence under conditionally mean-zero noise with a controlled conditional second moment; independence is not required.
\citet{karandikar2024convergence} allow biased noise and obtain almost sure rates under additional decay and Lyapunov conditions.
Their bias condition, written in our notation, requires
\begin{align}
\norm{\E[H(w_t,Y_{t+1})-h(w_t)\mid w_0,Y_1,\ldots,Y_t]}
\leq C_t(1+\norm{w_t}),
\qquad \sum_{t=0}^{\infty}\alpha_t C_t<\infty
\end{align}
(see Theorem~7.1, Equation~(69), of \citealt{karandikar2024convergence}).
For the one-step Markovian noise in~\eqref{eq sa update}, the conditional mean need not vanish or decay with $t$: geometric mixing controls the effect of a growing time lag, whereas the conditioning here is only one step earlier.
Thus these noise conditions are not guaranteed by our Markovian assumptions, and the cited results do not directly cover our setting in general.


\tb{High Probability Concentration Bounds.}
For high probability concentration bounds, SGD is also a main focus of prior works, including \citet{rakhlin2012making,duchi2012ergodic,hazan2014beyond,harvey2019simple,zhu2022beyond,telgarsky2022stochastic}.
We refer the reader to \citet{chen2025concentration} for a more detailed review of those works since SGD is not the main focus of this work.

For general stochastic approximation and reinforcement learning, \citet{even2003learning} is perhaps the earliest work that establishes a concentration bound with exponential tails.
They analyze both synchronous and asynchronous $Q$-learning.
Table~\ref{tbl sa works concentration} reports their synchronous results (Theorems~2--3), which use parallel sampling and a stepsize indexed by the global iteration count.
Their asynchronous results (Theorems~4--5) instead use stepsizes based on state-action visit counts and impose a covering-time condition on the sampling trajectory.
The synchronous results give high-probability bounds at a specified iteration, and their analysis uses the boundedness of the $Q$-learning iterates and noise (Remark~\ref{rem soft q vs q}).
\citet{qu2020finite} study general stochastic approximation and use $Q$-learning as an application.
In their general results,
it is assumed that the iterates are bounded.
In their $Q$-learning application,
this assumption is fulfilled via using specific properties of $Q$-learning (Remark~\ref{rem soft q vs q}).
\localrevision{\citet{li2021tightening} study synchronous $Q$-learning with i.i.d. sampling.}
Furthermore, they also require specific properties of $Q$-learning to establish the boundedness first (Remark~\ref{rem soft q vs q}).
\citet{li2024q} extend the analysis of $Q$-learning in \citet{li2021tightening} from synchronous setup to asynchronous setup and thus allow Markovian noise.
However,
their concentration bound only holds for the last iterate, and their constant learning rate depends on the prefixed number of iterates.
\localrevision{\citet{li2020sample} give high-probability sample-complexity bounds for asynchronous $Q$-learning along a single Markovian trajectory, including a constant-stepsize analysis and a variance-reduced variant. These finite-time guarantees complement the almost sure rates studied here.}
\citet{prashanth2021concentration} study a batch version of TD with linear function approximation,
using a projection operator to ensure the boundedness of iterates.
Furthermore,
their algorithm runs on a given dataset, so their noise is essentially i.i.d.
\citet{thoppe2019concentration} establish a concentration bound for general stochastic approximation algorithms under i.i.d. noise.
But their bound holds only when $t \geq t_0$ with $t_0$ depending on some unknown random event (the iterates run into the domain of attraction of $w_*$).
This work is improved by \citet{chandak2022concentration} to allow both i.i.d. and Markovian noise and to remove the dependence on the unknown event.
However,
it relies on the a priori assumption that the Lipschitz constant of $w \mapsto H(w, y)$ (i.e., $L_h$ in Assumption~\ref{assu Lipschitz}) is strictly smaller than 1.
Although requiring $h(w)$ to be contractive is standard,
requiring $H(w, y)$ to be contractive for each $y$ is restrictive.
The follow-up work \citet{chandak2026concentration} lifts this contraction requirement of $H(w, y)$ in the context of linear TD.
However,
to make the tail of their results exponential,
they have to rely on \citet{korda2015td} to bound a term.
\citet{dalal2018finite} study linear TD with i.i.d. noise.
Their bound holds only for $t \geq t_0$ with $t_0$ in the order of $\ln(1/\delta)$.
\citet{dalal2018finitesample} extend \citet{dalal2018finite} and study general two timescale linear stochastic approximation under i.i.d. noise.
Similarly, their $t_0$ also depends on $\delta$.
\citet{durmus2021tight} study linear stochastic approximation with i.i.d. noise
using a constant learning rate. 
\citet{mou2022optimal} apply variance reduction techniques to general stochastic approximation algorithms with i.i.d. noise 
(i.e., instead of using $H(w_t, Y_{t+1})$, they use an empirical average of $H(w_t, Y_{t+1})$).
Their bound holds only for the last iterate and requires a constant learning rate.
Among earlier concentration analyses, \citet{chen2025concentration} is closely related to our approach.
In fact,
we will apply the techniques therein on~\eqref{eq skeleton sa} to obtain concentration bounds for Markovian $\qty{Y_t}$.
By contrast,
\citet{chen2025concentration} work on the original iterates~\eqref{eq sa update} directly under conditionally mean-zero noise.

Subsequent to our 2024 preprint, \citet{agrawal2026concentration} established concentration results for finite-state Markovian SA under harmonic and other polynomial stepsizes.
Under its noise, Lyapunov, and stepsize conditions, their Theorem~2.1 gives an all-time bound with squared-error decay of order $1/t$ up to logarithmic factors under harmonic stepsizes.
Our Theorem~\ref{thm concentration} uses a log-damped harmonic stepsize \ref{assu lr2} 
and the corresponding sub-polynomial rate is intrinsic to our stepsize (Proposition~\ref{prop lr2 lower}).
As a result, \citet{agrawal2026concentration} is better understood as broadening the class of stepsizes and settings for which concentration results can be obtained, rather than directly competing with our log-damped harmonic stepsize analysis to provide sharper rates.

\tb{$L^p$ Convergence Rates.} 
$L^2$ convergence rates for stochastic approximation with Markovian noise is a very large research area and is not the main topic of this work.
So we skip a detailed survey and only mention a few works such as
\citet{bhandari2018finite,zou2019finite,srikant2019finite,chen2024lyapunov,zhang2022global,zhang2022truncated,kotsalis2022simple,haque2025stochastic,liu2025linearq,xie2025finite,liu2026extensions}.
Since $L^p$ convergence can be derived via integrating exponential tails once a high probability concentration is obtained for each $t$,
here we only survey works that obtain $L^p$ convergence rates in other more direct methods.
For $L^p$ convergence rates of SGD,
we refer the reader to \citet{godichon2016estimating,godichon2019lp}.
The literature on general $L^p$ convergence for general nonlinear stochastic approximation with Markovian noise is rather thin.
To our knowledge, \citet{lauand2024revisiting} is the only work that falls into this category.
However, \citet{lauand2024revisiting} only establish asymptotic convergence in $L^p$, and there is no convergence rate.

\tb{Stochastic Approximation.}
In this work,
the Markovian noise $\qty{Y_t}$ is assumed to mix geometrically (Assumption~\ref{assu markov chain}), akin to~\citet{borkar2025ode}.
In \citet{liu2025ode},
this geometrically mixing assumption is lifted and \citet{liu2025ode} only need LLN to hold for $\qty{Y_t}$ to establish the almost sure convergence.
To establish an almost sure convergence rate and concentration bound under these weakened assumptions on $\qty{Y_t}$ is left for future work. 
Furthermore,
\citet{borkar2025ode,liu2025ode} also allow the Lipschitz constant $L_h$ in Assumption~\ref{assu Lipschitz} to depend on $y$ for their almost sure convergence.
To establish an almost sure convergence rate and a concentration bound under this weakened assumption on the Lipschitz continuity is also left for future work.

Our~\eqref{eq skeleton sa} technique is inspired by the proof of Proposition 4.8 of \citet{bertsekas1996neuro},
where the almost sure convergence of~\eqref{eq linear sa} is studied.
That proof uses a \tb{constant} $T$ to discretize the ODE and assumes $b(y)$ in~\eqref{eq linear sa} is always 0.
By contrast,
we use a diminishing $\qty{T_m}$ to discretize the ODE and study general contractive stochastic approximation algorithms,
establishing both an almost sure convergence rate (and thus almost sure convergence) and a maximal concentration bound with sub-Weibull tails.
The diminishing nature of $\qty{T_m}$ is vital to our success.
To see this,
we note that their proof is conducted under the assumption that $b(y) = 0$ 
and at the beginning of their proof,
it is claimed that assuming $b(y) = 0$ is only for simplifying presentations and the proof for a general $b(y)$ is similar (the proof with a general $b(y)$ is not provided).
We argue that this is not the case.
In particular,
their almost sure convergence in Proposition 4.8 is essentially obtained by a supermartingale convergence theorem (Proposition 4.2 of \citealt{bertsekas1996neuro} or Lemma~\ref{lem almost smg}).
When $b(y) = 0$ holds,
their proof is able to identify
the $\qty{B_t}$ in Lemma~\ref{lem almost smg} as 0.
So $\sum B_t = 0 < \infty$ almost surely (Condition (ii) of Lemma~\ref{lem almost smg}). 
But with a constant $T$ and a general $b(y)$,
the $\qty{B_t}$ in Lemma~\ref{lem almost smg} will be identified as some small positive constant in the order of $T^2$.
Although this constant can be arbitrarily small,
it is always positive,
indicating $\sum B_t = \infty$, so the supermartingale convergence theorem will not apply.
By using a diminishing $\qty{T_m}$,
we are able to identify $\qty{B_t}$ in Lemma~\ref{lem almost smg} as $T_m^2$ and have $\sum B_t < \infty$.


\section{Proofs of Theorem~\ref{thm lil} and Theorem~\ref{thm concentration}}
\label{sec main proof}
Both proofs start with the same skeleton iterates technique and deviate after we obtain certain recursive bounds. 

\subsection{Skeleton Iterates}
We first formally define the anchors for constructing the~\eqref{eq skeleton sa}.
We define
\begin{align}
  \label{eq def big tm}
  \textstyle T_m = \frac{C_\alpha \ln^{\beta_1} (m+3)}{(m+3)^{\beta_2}}.
\end{align}
We set $\beta_1$ and $\beta_2$ according to the learning-rate regime and, under \ref{assu lr}, the exponent $\nu$.
In particular, we set
\begin{align}
    \label{eq def eta}
  \begin{cases}
    \textstyle 0 < \beta_1 < 1,\, \beta_2 = 1 &\qq{for Assumption~\ref{assu lr} with $\nu = 1$ } \\
    \textstyle \beta_1 = 0,\, \frac{1}{2} < \beta_2 < \frac{\nu}{2 - \nu} &\qq{for Assumption~\ref{assu lr} with $\nu \in (2/3, 1)$} \\
    \textstyle \beta_1 = 0, \beta_2 = 1 &\qq{for Assumption~\ref{assu lr2}}
  \end{cases}.
\end{align}
For all three setups, $\qty{T_m}$ is always monotonically decreasing. We then define a sequence $\qty{t_m}$ as $t_0 = 0$,
\begin{align}
  \label{eq def tm}
  \textstyle t_{m+1} \doteq \textstyle\min\qty{k | \sum_{t=t_m}^{k-1} \alpha_t \geq T_m}, \, m=0,1,\dots
\end{align}
Under this definition, we have $t_{m+1} \geq t_m + 1$ so $\lim_{m\to\infty} t_m = \infty$. More importantly, 
\begin{align}
  \label{eq bar alpha m lower bound}
  \textstyle \bar \alpha_m \doteq \sum_{t=t_m}^{t_{m+1} - 1}\alpha_t \geq T_m.
\end{align}
We have now divided the real axis into intervals of length $\qty{\bar \alpha_m}$,
using $\qty{t_m}$ as anchors.
We now investigate the iterates $\qty{w_t}$ interval by interval.
To simplify the presentation,
from now on, we define shorthand
\begin{align}
  G(w, y) \doteq& H(w, y) - w, \quad g(w) \doteq h(w) - w.
\end{align}
Clearly $w \mapsto G(w, y)$ is $L_h + 1$ Lipschitz.
Telescoping~\eqref{eq sa update} yields
\begin{align}
  w_{t_{m+1}} =& \textstyle w_{t_m} + \sum_{t=t_m}^{t_{m+1}-1} \alpha_t G(w_t, Y_{t+1}) \\
  =& \textstyle w_{t_m} + \sum_{t=t_m}^{t_{m+1}-1} \alpha_t \qty(g(w_{t_m}) + G(w_t, Y_{t+1}) - g(w_{t_m})) \\
  =& w_{t_m} + \bar \alpha_m g(w_{t_m}) + \underbrace{\textstyle \sum_{t=t_m}^{t_{m+1}-1} \alpha_t \qty(G(w_t, Y_{t+1}) - g(w_{t_m}))}_{z_m}.
\end{align}
In other words,
$\qty{w_{t_m}}$ can be regarded as iterates generated by a new stochastic approximation algorithm \eqref{eq skeleton sa} running in the timescale $\qty{t_m}$ while the original algorithm~\eqref{eq sa update} runs in the timescale $\qty{t}$.
The following properties of the intervals will be used repeatedly.
\begin{lemma}
    \label{lem lr bounds}
    There exists some $C_\tref{lem lr bounds}$ and $m_0$ such that 
    for all $m \ge m_0$ and $t \ge t_m$, we have $\alpha_t \le C_\tref{lem lr bounds}T_m^2$.
\end{lemma}
The proof is in Section~\ref{sec proof lem lr bounds}.
Notably,
Lemma~\ref{lem lr bounds} is the key result to balance $\alpha_t$ and $T_m$.

\begin{lemma}
    \label{lem lr bounds 2}
    There exists some $m_0$ such that
    for all $m \ge m_0$, $\bar \alpha_m \le T_m + C_\tref{lem lr bounds}T_m^2 \leq 2T_m$.
\end{lemma}
The proof is in Section~\ref{sec proof lem lr bounds 2}. 
Lemma~\ref{lem lr bounds 2} and~\eqref{eq bar alpha m lower bound} confirm that $\bar \alpha_m = \Theta\qty(T_m)$.
Notably, many lemmas below will need ``$m \geq m_0$''.
For simplifying notations, we assume the current $m_0$ is large enough such that all previous lemmas hold.
For example, the $m_0$ in Lemma~\ref{lem lr bounds 2} is assumed to be large enough such that Lemma~\ref{lem lr bounds} also holds.
But $m_0$ always remains deterministic,
independent of sample path and probability parameter $\delta$.

\subsection{Moreau Envelope-Based Lyapunov Function}

We now construct a Lyapunov function to study the behavior of $\qty{w_{t_m}}$.
The most straightforward candidate would be $\norm{w - w_*}^2$.
Unfortunately, the function $\norm{w - w_*}^2$ is not necessarily smooth,
rendering the analysis challenging.
Following \citet{chen2020finite,chen2024lyapunov,zhang2022global},
we then construct the Moreau envelope of $\norm{\cdot}^2$ as the Lyapunov function.
We say a function $f: \R^d \to \R$ is $L$-smooth w.r.t. some norm $\norm{\cdot}_s$ if 
\begin{align}
  \label{eq smooth def}
 \textstyle  f(w') \leq f(w) + \indot{\nabla f(w)}{w' - w} + \frac{L}{2} \norm{w'-w}_s^2.
\end{align}
Fix $\norm{\cdot}_s$ to be an arbitrary norm such that $\frac{1}{2}\norm{\cdot}_s^2$ is $L$-smooth, e.g., $\ell_p$ norm with $p \geq 2$.
The Moreau envelope for $\frac{1}{2}\norm{\cdot}^2$ w.r.t. $\frac{1}{2}\norm{\cdot}_s^2$ is defined as
\begin{align}
 \textstyle  M(w) \doteq \inf_{u \in \R^d}\qty{\frac{1}{2} \norm{u}^2 + \frac{1}{2 \xi} \norm{w - u}_s^2},
\end{align}
where $\xi > 0$ is a constant to be tuned.
Due to the equivalence between norms,
there exist positive constants $l_{cs}$ and $u_{cs}$ such that 
  $l_{cs} \norm{w}_s \leq \norm{w} \leq u_{cs} \norm{w}_s$.
The properties of $M(w)$ are summarized below.
\begin{lemma}
  (Proposition A.1 and Section A.2 of \citealt{chen2024lyapunov})
  \label{lem property of M}
  \begin{enumerate}[(i).]
      \item $M(w)$ is $\frac{L}{\xi}$-smooth w.r.t. $\norm{\cdot}_s$.
      \item There exists a norm $\norm{\cdot}_M$ such that $M(w) = \frac{1}{2}\norm{w}_M^2$.
      \item Define
          $l_{cm} = \sqrt{(1 + \xi l_{cs}^2)}$,
          $u_{cm} = \sqrt{(1 + \xi u_{cs}^2)}$,
      then 
          $l_{cm}\norm{w}_M \leq \norm{w} \leq u_{cm} \norm{w}_M$.
      \item $\indot{\nabla M(w)}{w'} \leq \norm{w}_M \norm{w'}_M, \, \indot{\nabla M(w)}{w} \geq \norm{w}_M^2$.
  \end{enumerate}
\end{lemma}

\subsection{Noise Decomposition and Bounds}
In view of Lemma~\ref{lem property of M}, realizing $f, w', w$ in~\eqref{eq smooth def} as $f = M$, $w' = w_{t_{m+1}} - w_*$, and $w = w_{t_m} - w_*$
yields
\begin{align}
  \label{eq main l smooth}
  &\norm{w_{t_{m+1}} - w_*}_M^2 \\
  \leq& \textstyle \norm{w_{t_m} - w_*}_M^2 + \indot{\nabla \norm{w_{t_m} - w_*}_M^2}{\bar \alpha_m g(w_{t_m}) + z_m} + \frac{L}{\xi} \norm{\bar \alpha_m g(w_{t_m}) + z_m}_s^2.
\end{align}
Let $\fF_t \doteq \sigma(w_0, Y_1, \dots, Y_t)$ be the filtration until time $t$ and use $\E_m\qty[\cdot] \doteq \E\qty[\cdot \mid \fF_{t_m}]$ to denote the conditional expectation given $\fF_{t_m}$.
We then perform decomposition 
\begin{align}
  \textstyle z_m = \sum_{t=t_m}^{t_{m+1}-1} \alpha_t \qty(G(w_t, Y_{t+1}) - g(w_{t_m})) = z_{1, m} + z_{2, m} + z_{3, m}
\end{align}
where
\begin{align}
 z_{1, m} \doteq& \textstyle \sum_{t=t_m}^{t_{m+1}-1} \alpha_t \qty(G(w_t, Y_{t+1}) - G(w_{t_m}, Y_{t+1})), \\
 z_{2, m} \doteq& \textstyle  \sum_{t=t_m}^{t_{m+1}-1} \alpha_t \qty(G(w_{t_m}, Y_{t+1}) - \E_m\qty[G(w_{t_m}, Y_{t+1})]), \\
 z_{3, m} \doteq& \textstyle   \sum_{t=t_m}^{t_{m+1}-1} \alpha_t \qty(\E_m\qty[G(w_{t_m}, Y_{t+1})] - g(w_{t_m})).
\end{align}
Notably, $\E_m\qty[z_{2, m}] = 0$.
We then rearrange~\eqref{eq main l smooth} as
\begin{align}
  \label{eq main smooth ineq}
  &\norm{w_{t_{m+1}} - w_*}_M^2 \\
  \leq& \textstyle  \norm{w_{t_m} - w_*}_M^2 + \bar \alpha_m \indot{\nabla \norm{w_{t_m} - w_*}_M^2}{g(w_{t_m})} + \indot{\nabla \norm{w_{t_m} - w_*}_M^2}{z_{2, m}} \\
  & \textstyle + \indot{\nabla \norm{w_{t_m} - w_*}_M^2}{z_{1, m} + z_{3, m}} + \frac{L}{\xi} \norm{\bar \alpha_m g(w_{t_m}) + z_{1, m} + z_{2, m} + z_{3, m}}_s^2.
\end{align}
We pause to clarify the purpose of this decomposition. The term $z_{1,m}$ is the within-block discretization error caused by replacing $w_t$ with the anchor iterate $w_{t_m}$. The term $z_{2,m}$ is the centered fluctuation term; it satisfies
$\mathbb{E}[z_{2,m}\mid \mathcal{F}_{t_m}]=0$ and will be treated as the martingale-type noise in the skeleton recursion.
The term $z_{3,m}$ is the Markovian bias term, measuring
the difference between the conditional distribution of $Y_{t+1}$ and its stationary distribution. 
Lipschitz continuity controls $z_{1,m}$, geometric mixing controls $z_{3,m}$, and only $z_{2,m}$ remains as the leading stochastic term. This yields the recursive Lyapunov bound in Lemma~\ref{lem bound all}.
We now bound the terms of the RHS one by one.

\begin{lemma}
  \label{lem bound leading term}
    $\indot{\nabla \norm{w_{t_m} - w_*}_M^2}{g(w_{t_m})} \leq -2(1 - \frac{u_{cm}}{l_{cm}}\kappa) \norm{w_{t_m} - w_*}_M^2$.
\end{lemma}
The proof is in Section~\ref{sec proof lem bound leading term}.
By Lemma~\ref{lem property of M},
it is easy to see $\lim_{\xi \to 0} \frac{u_{cm}}{l_{cm}} = 1$.
From now on,
we fix a sufficiently small $\xi$ such that 
\begin{align}
 \textstyle  \kappa' \doteq 1 - \frac{u_{cm}}{l_{cm}}\kappa \in (0, 1].
\end{align}
This also allows us to identify a sufficient threshold for Theorem~\ref{thm concentration}:
\begin{align}
\label{eq concentration threshold}
\textstyle\bar C_\alpha = \frac{2}{\kappa'} = \frac{2}{1-(u_{cm}/l_{cm})\kappa}.
\end{align}
This sufficient choice ensures $C_\alpha\kappa'>1$, the strict inequality used by the concentration bootstrap.

Since $u_{cm}/l_{cm}\to1$ as $\xi\to0$, we can further choose $\xi$ so that
$(u_{cm}/l_{cm})\kappa\leq(1+\kappa)/2$. Then
$(1-\kappa)/2\leq\kappa'\leq1-\kappa$, which yields the stated bounds on $\bar C_\alpha$.
The quadratic remainder in~\eqref{eq main smooth ineq} satisfies
\begin{align}
\textstyle\frac{L}{\xi}\norm{u}_s^2
\leq\frac{Lu_{cm}^2}{\xi l_{cs}^2}\norm{u}_M^2.
\end{align}
For example, if $\norm{\cdot}=\norm{\cdot}_\infty$ on $\mathbb R^d$ and
$\norm{\cdot}_s=\norm{\cdot}_2$, we may take $l_{cs}=1/\sqrt d$, $u_{cs}=1$, and $L=1$.
Choosing $\xi=1-\kappa$ gives
$(u_{cm}/l_{cm})\kappa\leq\kappa\sqrt{1+\xi}\leq\kappa(1+\xi/2)\leq(1+\kappa)/2$
and
\begin{align}
\label{eq dimension curvature}
\textstyle\frac{Lu_{cm}^2}{\xi l_{cs}^2}
=\frac{d(2-\kappa)}{1-\kappa}
\leq\frac{2d}{1-\kappa}.
\end{align}
This becomes one admissible choice of smoothing norm and parameter.


\begin{lemma}
  \label{lem bound z1}
  There exists some deterministic $C_\tref{lem bound z1}$ and $m_0$ such that for all $m \geq m_0$, 
  \begin{align}
  \norm{z_{1, m}}_M \leq T_m^2 C_\tref{lem bound z1} \qty(\norm{w_{t_m} - w_*}_M + 1).
  \end{align}
\end{lemma}
The proof is in Section~\ref{sec proof lem bound z1}.

\begin{lemma}
  \label{lem bound z2}
  There exists some deterministic $C_\tref{lem bound z2}$ and $m_0$ such that for all $m \geq m_0$, 
  \begin{align}
  \norm{z_{2, m}}_M \leq T_m C_\tref{lem bound z2} \qty(\norm{w_{t_m} - w_*}_M + 1).
  \end{align}
\end{lemma}
The proof is in Section~\ref{sec proof lem bound z2}.

\begin{lemma}
  \label{lem bound z3}
  There exists some deterministic $C_\tref{lem bound z3}$ and $m_0$ such that for all $m \geq m_0$, 
  \begin{align}
    \norm{z_{3, m}}_M \leq T_m^2C_\tref {lem bound z3} (\norm{w_{t_m} - w_*}_M + 1).
  \end{align}
\end{lemma}
The proof is in Section~\ref{sec proof lem bound z3}.
Assembling the above bounds yields
\begin{lemma}
  \label{lem bound all}
  There exists some deterministic $C_\tref{lem bound all}$ and $m_0$ such that for all $m \geq m_0$, 
  \begin{align}
    \norm{w_{t_{m+1}} - w_*}_M^2 \leq (1 - T_m \kappa')\norm{w_{t_m} - w_*}_M^2+ \indot{\nabla \norm{w_{t_m} - w_*}_M^2}{z_{2,m}} + C_\tref{lem bound all} T_m^2.
  \end{align}
\end{lemma}
The proof is in Section~\ref{sec proof lem bound all}.

Notably, our error decomposition of $z_m$ is significantly different from previous works with Markovian noise like \citet{bhandari2018finite,zou2019finite,chen2024lyapunov,zhang2022global}.
In those works,
they work on the original timescale $\qty{t}$ in~\eqref{eq sa update} and
$G(w_t, Y_{t+1})$ is usually decomposed as
\begin{align}
  G(w_t, Y_{t+1}) - g(w_t) =& G(w_{t-\tau_t}, Y_{t+1}) - g(w_{t-\tau_t}) \\
  &+G(w_t, Y_{t+1}) - G(w_{t-\tau_t}, Y_{t+1}) \\
  &+ g(w_{t-\tau_t}) - g(w_t),
\end{align}
where $\tau_t$ is a function of $t$ depending on the mixing rate $\varrho$ in Assumption~\ref{assu markov chain} and the learning rate $\alpha_t$.
The last two terms in the RHS can be easily bounded via Lipschitz continuity.
To bound the first term in the RHS,
they then have to take $\E\qty[\cdot \mid \fF_{t-\tau_t}]$ on their counterparts of~\eqref{eq main smooth ineq}.
This will relate $\E\qty[\norm{w_{t+1} - w_*}^2_M \mid \fF_{t-\tau_t}]$ to $\E\qty[\norm{w_t - w_*}^2_M \mid \fF_{t - \tau_t}]$.
But this is not a recursion since the conditional expectations are conditioned on the same $\fF_{t-\tau_t}$.
And $w_t$ is not adapted to $\fF_{t-\tau_t}$.
So they have to take another expectation over $\fF_{t-\tau_t}$ and thus get recursions between $\E\qty[\norm{w_{t+1} - w_*}^2_M]$ and $\E\qty[\norm{w_t - w_*}^2_M]$.
This recursion is enough to obtain an $L^2$ convergence rate but does not help obtain an almost sure convergence rate.
By contrast,
our Lemma~\ref{lem bound all} will be used to connect $\E\qty[\norm{w_{t_{m+1}} - w_*}^2_M \mid \fF_{t_m}]$ with $\norm{w_{t_m} - w_*}^2_M$ for Theorem~\ref{thm lil} and to connect $\E\qty[\exp(\norm{w_{t_{m+1}} - w_*}^2_M) \mid \fF_{t_m}]$ with $\exp(\norm{w_{t_m} - w_*}^2_M)$ for Theorem~\ref{thm concentration},
both of which resemble a supermartingale.

\subsection{Almost Sure Convergence Rates (Proof of Theorem~\ref{thm lil})}
We recall that we proceed under Assumption~\ref{assu lr}
and $\E_m\qty[z_{2, m}] = 0$.
Taking conditional expectation on both sides of Lemma~\ref{lem bound all} then yields
\begin{align}
  \label{eq lil martingale}
  \textstyle\E_m\qty[\norm{w_{t_{m+1}} - w_*}_M^2] \leq (1 - T_m \kappa')\norm{w_{t_m} - w_*}_M^2 + C_\tref{lem bound all} T_m^2.
\end{align}
This ensures that $\qty{\norm{w_{t_m} - w_*}_M^2}$ is ``almost'' a supermartingale w.r.t. $\qty{\fF_{t_m}}$.
As a result,
a canonical result in \citet{robbins1971convergence} regarding the convergence of ``almost'' supermartingale confirms that $\lim_{m\to\infty} \norm{w_{t_m} - w_*}_M^2 = 0$ almost surely.
A finer version of this canonical result in \citet{liu2024almost,karandikar2024convergence} (included as Lemma~\ref{lem almost smg} for completeness)
further yields an almost sure convergence rate.
\begin{lemma}
  \label{lem as rate skeleton sa}
  For any $\nu \in (2/3, 1]$ and $\epsilon \in (0, 2\beta_2 - 1)$, it holds that
\begin{align}
\textstyle\lim_{m\to\infty} m^\epsilon \norm{w_{t_m} - w_*}_M^2 = 0 \quad \text{a.s.}
\end{align}
\end{lemma}
The proof is in Section~\ref{sec proof lem as rate skeleton sa}. 
We recall that $\beta_2$ is chosen according to~\eqref{eq def eta}.
For $\nu = 1$,
we have $\beta_2 = 1$.
For $\nu<1$ sufficiently close to $1$, $\beta_2$ can be chosen arbitrarily close to $1$.
Thus, under Assumption~\ref{assu lr}, $\epsilon$ can be chosen arbitrarily close to $1$
by taking $\nu=1$ or choosing $\nu<1$ sufficiently close to $1$.
Then the almost sure convergence rate in Lemma~\ref{lem as rate skeleton sa} can be arbitrarily close to $\sqrt{m} \norm{w_{t_m} - w_*}_M \to 0$,
which resembles a square-root rate in skeleton time. This observation does not establish optimality: the skeleton noise and block lengths vary with $m$, and converting from $m$ to the original index $t$ loses rate, as discussed following Theorems~\ref{thm lil} and~\ref{thm concentration}.

Mapping the almost sure convergence rate from $m$ back to $t$ (Section~\ref{sec proof lem connect skeleton sa and sa}) then completes the proof of Theorem~\ref{thm lil}.



\subsection{Concentration (Proof of Theorem~\ref{thm concentration})}
In this section, we proceed under Assumption~\ref{assu lr2}.
We will follow \citet{chen2025concentration} to establish the concentration bound,
which is an inductive (bootstrapping / recursive) argument.
We first lay out the inductive basis.
\begin{lemma}
  \label{lem inductive basis}
  There exist positive deterministic integers $C_\tref{lem inductive basis}$ and $C_\tref{lem inductive basis}'$
  such that for any $m \geq 0$,
  \begin{align}
      \norm{w_{t_m} - w_*}_M^2 \leq C_\tref{lem inductive basis}'(m+3)^{C_\tref{lem inductive basis}} \qq{a.s.}
  \end{align}
\end{lemma}
The proof is in Section~\ref{proof of lem inductive basis}.
Lemma~\ref{lem inductive basis} uses the discrete Gronwall inequality (Lemma~\ref{lem discrete_gronwall}) to give the worst-case almost sure bound. 
We now present the inductive step.
\begin{lemma}
  \label{lem inductive step} 
  There exists some deterministic constant $m_0$ such that the following implication relation holds.
  For any $\delta \in [0, 1)$, suppose there is a deterministic, non-decreasing sequence
  $\qty{B_m(\delta)}_{m=m_0}^\infty$ with $B_m(\delta)\geq1$ for every $m\geq m_0$, such that
  \begin{align}
    \label{eq prob ineq 1} 
    \Pr(\forall m \geq m_0, \e{m} \leq B_m(\delta)) \geq 1 - \delta,
  \end{align}
  Then for any $\delta' \in (0, 1-\delta)$,
  there exists a new sequence $\qty{B_m(\delta, \delta')}_{m=m_0}^\infty$ such that
  \begin{align}
    \Pr(\forall m \geq m_0, \e{m} \leq B_m(\delta, \delta')) \geq 1 - \delta - \delta',
  \end{align}
  where $B_m(\delta, \delta') = C_\tref{lem inductive step} T_m B_m(\delta)(1 +  \ln(1/\delta') +  \ln (m+3))$ with $C_\tref{lem inductive step}$ being some constant depending on $m_0$ and independent of $m, \delta, \delta'$.
\end{lemma}
The proof is in Section~\ref{proof of lem inductive step}.
In other words, by using Lemma~\ref{lem inductive step} once,
we can improve the bound by $T_m$.
Since the almost sure bound given in Lemma~\ref{lem inductive basis} is of the order $(m+3)^{C_\tref{lem inductive basis}}$,
after applying Lemma~\ref{lem inductive step} for $C_\tref{lem inductive basis} + 1$ times,
we will get the desired bound.
Section~\ref{proof of lem concentration in m} verifies the input conditions at each application. The last output tends to zero and is not used as another input.
\begin{lemma}
  \label{lem concentration in m} 
  There exists some constant $C_\tref{lem concentration in m} = C_\tref{lem inductive basis} + 1$ and $C_\tref{lem concentration in m}'$ such that for any $\delta \in (0,1)$, it holds, with probability at least $1-\delta$,
  that for all $m \geq 0$,
  \begin{align}
    \|w_{t_m}-w_*\|_M^2 \leq & \textstyle
    C_{\tref{lem concentration in m}}'\frac{1}{m+3} \left[\ln(\frac{C_{\tref{lem concentration in m}}}{\delta}) + 1 + \ln (m+3) \right]^{C_{\tref{lem concentration in m}}}.
  \end{align}
\end{lemma}
The proof is in Section~\ref{proof of lem concentration in m}.
We do note that the proofs of Lemmas~\ref{lem inductive step} and~\ref{lem concentration in m} are analogous to \citet{chen2025concentration}
and we include them here mostly for completeness.
The techniques therein are not our contribution.
Instead,
applying those techniques to the~\eqref{eq skeleton sa} to handle Markovian noise is our contribution. Our \ref{assu lr2} bound and the i.i.d. results of~\citet{chen2025concentration} use different stepsize schedules. As Proposition~\ref{prop lr2 lower} shows, no polynomial rate can be guaranteed in general under \ref{assu lr2}, even without noise; its sub-polynomial character is intrinsic to the stepsize schedule.

\begin{remark}
  The present bootstrap argument uses $T_m = \Theta(1/m)$.
  The polynomial envelope $(m+3)^{C_\tref{lem inductive basis}}$ in Lemma~\ref{lem inductive basis} follows from the Gronwall bound, since $\sum_{i=0}^m T_i=O(\ln(m+3))$.
  If, e.g., $T_m = \Theta(1/m^{\beta_2})$ is used for some $\beta_2 < 1$,
  then the term would be of the order $\exp(m^{1-\beta_2})$.
  But each inductive step can only reduce the bound by a factor of $m^{-\beta_2}$.
  Then, this inductive approach can never reach a diminishing rate.
  This restriction on $T_m$ further restricts our choice of $\alpha_t$.
\end{remark}
Mapping the concentration bound from $m$ back to $t$ (Section~\ref{sec map the concentration bound back}) then completes the proof of Theorem~\ref{thm concentration}.

\section{Conclusion}
This work establishes, to our knowledge, the first almost sure convergence rate and the first maximal concentration bound with sub-Weibull tails for general contractive stochastic approximation algorithms with Markovian noise,
generating state-of-the-art analysis for $Q$-learning and off-policy TD.
We envision that the skeleton iterates technique introduced in this work could serve as a new tool for analyzing stochastic approximation algorithms with Markovian noise.
A possible future work is to refine the choice of $\qty{T_m}$ to further improve the convergence rates.

\Needspace{6\baselineskip}
\acks{This work is supported in part by the US National Science Foundation under the awards III-2128019, SLES-2331904, and CAREER-2442098, the Commonwealth Cyber Initiative's Central Virginia Node under the award VV-1Q26-001, VV-1Q27-007, and VV-1Q27005, 4VA at UVA, a Cisco Faculty Research Award, and a Google Research Award.
}

\appendix

\section{Proofs in Section~\ref{sec main proof}}
We first summarize the notation and then provide the auxiliary results and detailed proofs.

\subsection{Notations}
Unless stated otherwise, constants denoted by $C_i$ are finite deterministic constants that may depend on fixed problem parameters and the initialization $w_0$, but not on the iteration indices,
the sample path, or the confidence parameter $\delta$.
The deterministic threshold $m_0$ may be enlarged from one lemma to the next so that all previously established bounds hold.

\begingroup
\small
\renewcommand{\arraystretch}{1.12}

\begin{longtable}{p{0.19\linewidth} p{0.75\linewidth}}
\hline
\textbf{Notation} & \textbf{Definition / interpretation} \\
\hline
\endfirsthead

\multicolumn{2}{l}{Table~\ref{tab:notation} (continued)}\\
\hline
\textbf{Notation} & \textbf{Definition / interpretation} \\
\hline
\endhead

\multicolumn{2}{l}{\textit{Stochastic approximation and assumptions}}\\*
\hline
\noalign{\nobreak}
$t$ &
Index of the original stochastic approximation iterate. \\

$w_t\in\mathbb{R}^d$ &
Iterate generated by
$w_{t+1}=w_t+\alpha_t(H(w_t,Y_{t+1})-w_t)$. \\

$Y_t\in\mathcal{Y}$ &
Markovian noise process. \\

$H(w,y)$ &
Sample update operator. \\

$d_\fY$ &
Unique stationary distribution of the Markov chain $\{Y_t\}$. \\

$P$, $P^n$ &
One-step and $n$-step transition kernels of $\{Y_t\}$. \\

$\varrho\in[0,1)$ &
Geometric mixing rate in Assumption~\ref{assu markov chain}. Smaller $\varrho$ corresponds
to faster mixing. \\

$C_\aref{assu markov chain}$ &
Prefactor in the geometric-mixing bound
$\int |P^n(y,y')-d_\fY(y')|\,dy'\leq
C_\aref{assu markov chain}\varrho^n$. \\

$h(w)$ &
Stationary mean operator,
$h(w)\doteq \mathbb{E}_{Y\sim d_\fY}[H(w,Y)]$. \\

$w_*$ &
Unique fixed point of the mean operator, i.e.,
$h(w_*)=w_*$. \\

$\kappa\in[0,1)$ &
Contraction modulus of $h$ under the norm $\|\cdot\|$. Smaller
$\kappa$ means stronger contraction. \\

$L_h$ &
Uniform Lipschitz constant of $H(\cdot,y)$ in Assumption~\ref{assu Lipschitz}. \\

$\|\cdot\|$ &
Norm with respect to which the mean operator $h$ is contractive. \\

$\alpha_t$ &
Learning rate of the original stochastic approximation algorithm. \\

$C_\alpha$ &
Positive scale parameter of the learning rate. \\

$\nu$, $\omega$ &
Learning-rate decay parameters, distinct from the skeleton-interval parameters $\beta_1,\beta_2$. Under \ref{assu lr},
$\alpha_t=C_\alpha/(t+3)^\nu$ with $\nu\in(2/3,1]$;
under \ref{assu lr2},
$\alpha_t=C_\alpha/((t+3)\ln^\omega(t+3))$ with $\omega\in(0,1)$. \\

$G(w,y)$, $g(w)$ &
Residual update maps
$G(w,y)\doteq H(w,y)-w$ and
$g(w)\doteq h(w)-w$. \\

\hline

\multicolumn{2}{l}{\textit{Skeleton construction}}\\
\hline
\noalign{\nobreak}

$m$ &
Index of a skeleton block. A single skeleton step aggregates multiple
original iterations. \\

$T_m$ &
Target length of skeleton block $m$ in ODE time,
\par{\begin{minipage}{\linewidth}
\begin{align}
T_m=
\frac{C_\alpha \ln^{\beta_1}(m+3)}
     {(m+3)^{\beta_2}}.
\end{align}
\end{minipage}}\par
Its diminishing nature makes the effective skeleton learning rate
diminishing. \\

$\beta_1,\beta_2$ &
Parameters controlling the decay of $T_m$. They are chosen as
\par{\begin{minipage}{\linewidth}
\begin{align}
\begin{cases}
0<\beta_1<1,\ \beta_2=1,
    & \text{\ref{assu lr} with }\nu=1,\\
\beta_1=0,\ 
\frac12<\beta_2<\frac{\nu}{2-\nu},
    & \text{\ref{assu lr} with }\nu\in(2/3,1),\\
\beta_1=0,\ \beta_2=1,
    & \text{\ref{assu lr2}}.
\end{cases}
\end{align}
\end{minipage}}\par
The restrictions balance the original stepsize with the block length. \\

$t_m$ &
Anchor index of skeleton block $m$, defined recursively by
\par{\begin{minipage}{\linewidth}
\begin{align}
t_{m+1}
=\min\left\{k:
\sum_{t=t_m}^{k-1}\alpha_t\geq T_m\right\},
\qquad t_0=0.
\end{align}
\end{minipage}}\par
Thus block $m$ contains original iterations
$t_m,\ldots,t_{m+1}-1$. \\

$\bar\alpha_m$ &
Actual ODE-time length, or effective learning rate, of skeleton block $m$:
\par{\begin{minipage}{\linewidth}
\begin{align}
\bar\alpha_m
\doteq
\sum_{t=t_m}^{t_{m+1}-1}\alpha_t.
\end{align}
\end{minipage}}\par
Lemmas~\ref{lem lr bounds}--\ref{lem lr bounds 2} show $\bar\alpha_m=\Theta(T_m)$. \\

$w_{t_m}$ &
Skeleton iterate obtained by observing the original stochastic
approximation process only at the anchors $\{t_m\}$. \\

$z_m$ &
Accumulated error/noise in skeleton block $m$, so that
\par{\begin{minipage}{\linewidth}
\begin{align}
w_{t_{m+1}}
=
w_{t_m}+\bar\alpha_m g(w_{t_m})+z_m.
\end{align}
\end{minipage}}\par \\

$z_{1,m}$ &
Within-block discretization error arising from replacing $w_t$ by
the anchor iterate $w_{t_m}$. Its bound is proportional to $T_m^2(\|w_{t_m}-w_*\|_M+1)$. \\

$z_{2,m}$ &
Centered fluctuation term. It satisfies
$\mathbb{E}[z_{2,m}\mid\mathcal F_{t_m}]=0$ and acts as the
martingale-type noise in the skeleton recursion. \\

$z_{3,m}$ &
Markovian bias caused by the difference between the conditional
distribution of $Y_{t+1}$ and its stationary distribution. Geometric
mixing bounds it by a constant times $T_m^2(\|w_{t_m}-w_*\|_M+1)$. \\

$\mathcal F_t$ &
Natural filtration
$\mathcal F_t\doteq\sigma(w_0,Y_1,\ldots,Y_t)$. \\

$\mathbb E_m[\cdot]$ &
Conditional expectation
$\mathbb E[\cdot\mid\mathcal F_{t_m}]$. \\

$C_{\tref{lem lr bounds}}$ &
Balance constant in Lemma~\ref{lem lr bounds} satisfying
$\alpha_t\leq C_{\tref{lem lr bounds}}T_m^2$ for sufficiently large $m$ and all
$t\geq t_m$. This is the key relation controlling the accumulated
within-block and Markovian errors. \\

$m_0$ &
A sufficiently large \emph{deterministic} skeleton index after which
the required technical inequalities hold. It is independent of the
sample path and $\delta$ and may be enlarged between lemmas. \\

\hline
\multicolumn{2}{l}{\textit{Lyapunov function and effective contraction}}\\
\hline
\noalign{\nobreak}

$\|\cdot\|_s$ &
Auxiliary ``smooth'' norm for which
$\frac12\|\cdot\|_s^2$ is smooth. It is introduced only to construct
a smooth Lyapunov function. \\

$L$ &
Smoothness constant of $\frac12\|\cdot\|_s^2$. \\

$l_{cs},u_{cs}$ &
Norm-equivalence constants satisfying
$l_{cs}\|w\|_s\leq\|w\|\leq u_{cs}\|w\|_s$. \\

$\xi>0$ &
Moreau-envelope smoothing parameter. It is chosen sufficiently small
so that the smoothed Lyapunov function preserves a positive
contraction margin. \\

$M(w)$ &
Moreau-envelope Lyapunov function
\par{\begin{minipage}{\linewidth}
\begin{align}
M(w)\doteq
\inf_{u\in\mathbb R^d}
\left\{
\frac12\|u\|^2+
\frac{1}{2\xi}\|w-u\|_s^2
\right\}.
\end{align}
\end{minipage}}\par \\

$\|\cdot\|_M$ &
Norm induced by the Moreau envelope through
$M(w)=\frac12\|w\|_M^2$. \\

$l_{cm},u_{cm}$ &
Norm-equivalence constants satisfying
$l_{cm}\|w\|_M\leq\|w\|\leq u_{cm}\|w\|_M$, with
\par{\begin{minipage}{\linewidth}
\begin{align}
l_{cm}=\sqrt{1+\xi l_{cs}^2},
\qquad
u_{cm}=\sqrt{1+\xi u_{cs}^2}.
\end{align}
\end{minipage}}\par \\

$\kappa'$ &
Effective contraction margin after Moreau smoothing,
\par{\begin{minipage}{\linewidth}
\begin{align}
\kappa'
\doteq
1-\frac{u_{cm}}{l_{cm}}\kappa>0.
\end{align}
\end{minipage}}\par
It approaches $1-\kappa$ as $\xi\to0$. \\

$\bar C_\alpha$ &
Lower threshold on the \ref{assu lr2} stepsize coefficient in Theorem~\ref{thm concentration}.
A sufficient choice is $\bar C_\alpha=2/\kappa'$, ensuring the strict margin $C_\alpha\kappa'>1$. \\

$C_{\tref{lem bound z1}},C_{\tref{lem bound z2}},C_{\tref{lem bound z3}}$ &
Deterministic constants in the bounds for
$z_{1,m}$, $z_{2,m}$, and $z_{3,m}$, respectively. \\

$C_{\tref{lem bound all}}$ &
Deterministic remainder constant in the key skeleton Lyapunov recursion
\par{\begin{minipage}{\linewidth}
\begin{align}
\begin{aligned}
\|w_{t_{m+1}}-w_*\|_M^2
&\leq
(1-\kappa'T_m)\|w_{t_m}-w_*\|_M^2
\\ &\quad +
\langle\nabla\|w_{t_m}-w_*\|_M^2,z_{2,m}\rangle
+C_{\tref{lem bound all}}T_m^2.
\end{aligned}
\end{align}
\end{minipage}}\par \\

\hline
\multicolumn{2}{l}{\textit{Rates and concentration}}\\
\hline
\noalign{\nobreak}

$\zeta$ &
Almost-sure rate exponent in Theorem~\ref{thm lil}. For $\nu<1$, any
$0<\zeta<\frac32\nu-1$ is admissible; for $\nu=1$, $\zeta\in(0,1)$. \\

$\epsilon$ &
Auxiliary skeleton-rate exponent used in Lemma~\ref{lem as rate skeleton sa} and the
supermartingale rate argument. \\

$\delta$ &
Failure probability in the concentration result. \\

$B_m(\delta)$ &
A deterministic, non-decreasing upper envelope for
$\|w_{t_m}-w_*\|_M^2$, with $B_m(\delta)\geq1$, used as input to Lemma~\ref{lem inductive step}. \\

$C_{\tref{lem inductive basis}},C'_{\tref{lem inductive basis}}$ &
Exponent and prefactor in the initial worst-case polynomial
almost-sure bound for the skeleton iterates. \\

$C_{\tref{lem inductive step}}$ &
Constant controlling one application of the concentration
bootstrap in Lemma~\ref{lem inductive step}. \\

$C_{\tref{lem concentration in m}}$ &
Number of bootstrap steps and exponent on the logarithmic factor in
Theorem~\ref{thm concentration}; the proof chooses $C_{\tref{lem concentration in m}}=C_{\tref{lem inductive basis}}+1$.
The resulting sub-Weibull shape for the squared error is $1/C_{\tref{lem concentration in m}}$. \\

$C_{\mathrm{Th}\tref{thm concentration}},C'_{\mathrm{Th}\tref{thm concentration}}$ &
Deterministic constants in the final concentration bound of
Theorem~\ref{thm concentration} after translating from skeleton time $m$ back to original
time $t$. \\

$s(t)$ &
Time scale $s(t)=\ln^{1-\omega}(t+1)/(1-\omega)$ used in
Corollaries~\ref{cor lr2 as} and~\ref{cor lp} and the discussion of \ref{assu lr2}. \\

$a(t),b(t)$ &
Functions in Corollary~\ref{cor lp}:
$a(t)=C_{\mathrm{Th}\tref{thm concentration}}e^{-s(t)}$ and
$b(t)=C'_{\mathrm{Th}\tref{thm concentration}}+s(t)$, with the latter constant
enlarged to be nonnegative if necessary. Thus $a(t)>0$ and $b(t)\geq0$. \\

$c$ &
Any fixed number in $(0,1)$ in the almost sure limit
$e^{c s(t)}\|w_t-w_*\|^2\to0$ of Corollary~\ref{cor lr2 as}. \\

$p$ &
Real parameter $p\geq1$ in Corollary~\ref{cor lp}; the moment order is $2p$.
The case $p=1$ gives the mean-square bound. \\

$\Gamma$ &
Gamma function, with $\Gamma(x)=\int_0^\infty u^{x-1}e^{-u}\,du$ for $x>0$,
used in the moment bound of Corollary~\ref{cor lp}. \\

\hline
\multicolumn{2}{l}{\textit{Reinforcement-learning applications}}\\
\hline
\noalign{\nobreak}

$\mathcal S,\mathcal A$ &
Finite state and action spaces. \\

$r,p,p_0,\gamma$ &
Reward function, transition kernel, initial-state distribution, and
discount factor of the MDP. \\

$\mu,\pi,d_\mu$ &
Behavior policy, target policy, and stationary state distribution
induced by $\mu$. \\

$D_\mu$ &
Diagonal matrix $\operatorname{diag}(d_\mu(s):s\in\fS)$ in the mean TD operator. \\

$d_{\min}^{SA},\gamma'$ &
Minimum stationary state-action probability
$d_{\min}^{SA}=\min_{s,a}d_\mu(s)\mu(a|s)$ and contraction modulus
$\gamma'=1-(1-\gamma)d_{\min}^{SA}$ of the mean asynchronous $Q$-learning operator. \\

$\kappa_{\mathrm{TD}}$ &
Contraction modulus $1-(1-\gamma)\min_s d_\mu(s)$ of the mean off-policy TD operator. \\

$\bop_*,\bop_\pi$ &
Bellman optimality operator and Bellman operator for policy $\pi$. \\

$q_*,v_\pi$ &
Optimal action-value function and value function of the target policy. \\

$q_t,v_t$ &
Q-learning and off-policy TD iterates, respectively. \\

$R_{\max}$ &
Maximum absolute reward, $R_{\max}=\max_{s,a}|r(s,a)|$. \\

$\rho_t$ &
Sampled importance ratio $\rho_t=\pi(A_t|S_t)/\mu(A_t|S_t)$.
In~\eqref{eq td}, it multiplies the reward and successor-value terms. \\

$\rho_{\max}$ &
Maximum importance-sampling ratio:
\par{\begin{minipage}{\linewidth}
\begin{align}
\rho_{\max}\doteq\max_{(s,a):\mu(a|s)>0}\frac{\pi(a|s)}{\mu(a|s)}.
\end{align}
\end{minipage}}\par
For~\eqref{eq td}, $\max\{1,\gamma\rho_{\max}\}$ bounds the sample-operator
Lipschitz constant and $\rho_{\max}R_{\max}$ bounds $\|H(0,y)\|_\infty$. \\
\hline
\caption{Summary of notation.}\label{tab:notation}\\
\end{longtable}
\endgroup
\subsection{Auxiliary Lemmas}
\begin{lemma}
  [Discrete Gronwall Inequality](Lemma 8 in Section 11.2 of \citealt{borkar2009stochastic})
\label{lem discrete_gronwall}
    For non-negative real sequences $\{x_n, n \geq 0\}$ and $\{a_n, n \geq 0\}$ and scalars  $C > x_0$, $L \geq 0$,  it holds
    \begin{align}
        \textstyle x_{n+1} \leq C + L\sum_{i=0}^n a_ix_i \quad \forall n
    \implies
        \textstyle x_{n+1} \leq C\exp({L\sum_{i=0}^n a_i}) \quad \forall n.
    \end{align}
\end{lemma}
The following lemma is Theorem~5.2 of
\citet{karandikar2024convergence}.

\begin{lemma}
  \label{lem almost smg}
    Let $\qty{A_t}$, $\qty{B_t}$, $\qty{Z_t}$, and $\qty{\alpha_t}$ be sequences of non-negative random variables adapted to a filtration $\qty{\fF_t}$.
    For some $\epsilon \in (0,1)$, assume it holds that:
    \begin{enumerate}[(i)]
        \item for all $t \ge 0$, $\E[Z_{t+1} \mid \fF_t] \le (1+A_t)Z_t + B_t - \alpha_t Z_t$,
        \item $\sum_{t=0}^\infty A_t < \infty$, $\sum_{t=0}^\infty B_t < \infty$, $\sum_{t=0}^\infty \alpha_t = \infty$ a.s.,
        \item there exists a $T \ge 1$ such that $\alpha_t \ge \frac{\epsilon}{t}$ a.s. for all $t \ge T$,
        \item $\sum_{t=1}^\infty (t+1)^\epsilon B_t < \infty$, $\sum_{t=1}^\infty \qty(\alpha_t - \frac{\epsilon}{t}) = \infty$ a.s.
    \end{enumerate}
    Then $\lim_{t\to\infty} t^\epsilon Z_t = 0$ a.s.
\end{lemma}

\begin{lemma}
  \label{lem bound G}
  There exists some constant $C_\tref{lem bound G}$ such that for any $w, y$,
  \begin{align}
    \norm{G(w, y)} \leq C_\tref{lem bound G}(\norm{w} + 1), \, \norm{G(w, y)}_M \leq C_\tref{lem bound G}(\norm{w}_M + 1).
  \end{align}
\end{lemma}
\begin{proof}
  Setting $w'=0$ in~\eqref{eq G Lipschitz} yields
    $\norm{H(w, y) - H(0, y)} \leq L_h \norm{w}$.
  So we have
  \begin{align}
    \textstyle\norm{G(w, y)} \leq (L_h + 1) \norm{w} + \norm{H(0, y)} \leq (L_h + 1) \norm{w} + \sup_y \norm{H(0, y)}.
  \end{align}
  Assumption~\ref{assu Lipschitz} and norm equivalence then complete the proof.
\end{proof}

\begin{lemma}
  \label{lem bound gronwall}
  There exists some constant $C_\text{\ref{lem bound gronwall}}$ such that for any $m$ and any $t \in [t_m, t_{m+1}]$, 
  \begin{align}
    \norm{w_t - w_*} \leq& \textstyle C_\text{\ref{lem bound gronwall}}(\bar \alpha_m  + \norm{w_{t_m} - w_*}), \label{eq ttm gronwall}\\
    \norm{w_t - w_{t_m}} \leq& \bar \alpha_m C_\text{\ref{lem bound gronwall}}\qty(\norm{w_{t_m} - w_*} + 1) \label{eq ttm gronwall 2}.
  \end{align}
  Furthermore, for any $t \geq 0$, it holds that
  \begin{align}
    \norm{w_t - w_*} \leq \textstyle(\norm{w_0 - w_*} + C_\tref{lem bound gronwall}\sum_{\tau=0}^{t-1} \alpha_\tau ) \exp(C_\tref{lem bound gronwall} \sum_{\tau=0}^{t-1} \alpha_\tau) \label{eq ttm gronwall 3}.
  \end{align}
\end{lemma}
\begin{proof}
  \textit{Proof of~\eqref{eq ttm gronwall}.}
  From~\eqref{eq sa update} we have
  \begin{align}
    \label{eq basic sa bound}
    \norm{w_{t+1} - w_*} \leq& \textstyle\norm{w_t - w_*} + \alpha_t \norm{G(w_t, Y_{t+1})} \\
    \leq& \textstyle\norm{w_t - w_*} + \alpha_t C_\tref{lem bound G} \qty(\norm{w_t - w_*} + \norm{w_*} + 1).
  \end{align}
  For any $t \in [t_m, t_{m+1}]$,
  we have, by telescoping, that
  \begin{align}
    \norm{w_t - w_*} \leq& \textstyle\norm{w_{t_m} - w_*} + \sum_{\tau={t_m}}^{t-1}\alpha_\tau C_\tref{lem bound G}\qty(\norm{w_*} + 1) + \sum_{\tau=t_m}^{t-1} \alpha_\tau C_\tref{lem bound G} \norm{w_{\tau} - w_*} \\
    \leq& \textstyle\qty(\norm{w_{t_m} - w_*} + \bar \alpha_m C_\tref{lem bound G}\qty(\norm{w_*} + 1)) + \sum_{\tau=t_m}^{t-1} \alpha_\tau C_\tref{lem bound G} \norm{w_{\tau} - w_*} \\
    \leq& \textstyle\qty(\norm{w_{t_m} - w_*} + \bar \alpha_m C_\tref{lem bound G}\qty(\norm{w_*} + 1)) \exp(\sum_{\tau=t_m}^{t-1} \alpha_\tau C_\tref{lem bound G}),
  \end{align}
  where the last inequality is from the discrete Gronwall inequality.
  It then holds that
  \begin{align}
    \norm{w_t - w_*} \leq \qty(\norm{w_{t_m} - w_*} + \bar \alpha_m C_\tref{lem bound G}\qty(\norm{w_*} + 1)) \exp(\bar \alpha_m C_\tref{lem bound G}).
  \end{align}
  By the minimality of $t_{m+1}$ in~\eqref{eq def tm},
  $\sum_{\tau=t_m}^{t_{m+1}-2}\alpha_\tau<T_m$, where the sum is zero if $t_{m+1}=t_m+1$.
  Since $\{T_m\}$ and $\{\alpha_t\}$ are non-increasing, for every $m\geq0$ we have
  \begin{align}
  \bar\alpha_m<T_m+\alpha_{t_{m+1}-1}\leq T_0+\alpha_0.
  \end{align}
  Thus a constant valid for every skeleton interval is
  \begin{align}
  C_\text{\ref{lem bound gronwall}}
  =\max\{1,C_\tref{lem bound G}(\norm{w_*}+1)\}
  \exp\bigl((T_0+\alpha_0)C_\tref{lem bound G}\bigr).
  \end{align}
  \par
  \textit{Proof of~\eqref{eq ttm gronwall 2}.}
  Similarly, for any $t \in [t_m ,t_{m+1}]$, it holds that
  \begin{align}
    \norm{w_{t+1} - w_{t_m}} \leq \norm{w_t - w_{t_m}} + \alpha_t C_\tref{lem bound G} \qty(\norm{w_t - w_{t_m}} + \norm{w_{t_m} - w_*} + \norm{w_*} + 1).
  \end{align}
  Telescoping yields
  \begin{align}
    \norm{w_t - w_{t_m}} \leq&\textstyle \sum_{\tau=t_m}^{t-1} \alpha_\tau C_\tref{lem bound G} \qty(\norm{w_{t_m} - w_*} + \norm{w_*} + 1) + \sum_{\tau = t_m}^{t-1} \alpha_\tau C_\tref{lem bound G} \norm{w_\tau - w_{t_m}} \\
    \leq&\textstyle \bar \alpha_m C_\tref{lem bound G} \qty(\norm{w_{t_m} - w_*} + \norm{w_*} + 1) + \sum_{\tau = t_m}^{t-1} \alpha_\tau C_\tref{lem bound G} \norm{w_\tau - w_{t_m}} \\
    \leq&\textstyle \bar \alpha_m C_\tref{lem bound G} \qty(\norm{w_{t_m} - w_*} + \norm{w_*} + 1) \exp(\bar \alpha_m C_\tref{lem bound G}), 
  \end{align}
  where the last inequality is again from the discrete Gronwall inequality.
  Taking the chosen $C_\text{\ref{lem bound gronwall}}$ from above completes the proof. 
  \par
  \textit{Proof of~\eqref{eq ttm gronwall 3}.}
  For any $t \geq 0$,
  telescoping~\eqref{eq basic sa bound} from $t = 0$ yields
  \begin{align}
    \textstyle\norm{w_t - w_*} \leq \norm{w_0 - w_*} + \sum_{\tau =0}^{t-1} \alpha_\tau C_{\tref{lem bound G}} (\norm{w_\tau - w_*} + \norm{w_*} + 1).
  \end{align}
  Similarly, the discrete Gronwall inequality yields that for any $t$,
  \begin{align}
    \textstyle\norm{w_t - w_*} \leq (\norm{w_0 - w_*} + \sum_{\tau=0}^{t-1} \alpha_\tau C_\tref{lem bound G} (\norm{w_*} + 1)) \exp(C_\tref{lem bound G} \sum_{\tau=0}^{t-1} \alpha_\tau),
  \end{align}
  which completes the proof.
\end{proof}

\subsection{Proof of Lemma~\ref{lem lr bounds}}
\label{sec proof lem lr bounds}
\begin{proof}
  \tb{Case 1: $\nu < 1$ in Assumption~\ref{assu lr}.}
  In this case,
  we have $T_m = \Theta(m^{-\beta_2})$.
  From~\eqref{eq bar alpha m lower bound}, we have for any $m$,
  \begin{align}
    \textstyle \sum_{t=0}^{t_{m+1} - 1} \alpha_t \geq \sum_{i=0}^m T_i. 
  \end{align}
  Using integral approximation, we get
    $\sum_{i=0}^m T_i = \Theta\qty((m+1)^{1-\beta_2})$ and
     $\sum_{t=0}^{t_{m+1} - 1} \alpha_t = \Theta\qty(t_{m+1}^{1- \nu})$.
  This means $m^{1-\beta_2} = \fO(t_m^{1-\nu})$, i.e., $t_m^{-(1-\nu)} = \fO(m^{-(1-\beta_2)})$.
  For any $t \geq t_m$, we have
  \begin{align}
    \textstyle\alpha_t \leq \alpha_{t_m} = \fO\qty(t_m^{-\nu}) = \fO\qty({m^{-\frac{1-\beta_2}{1-\nu} \nu}}).
  \end{align}
  We recall that $T_m^2 = \Theta(m^{-2\beta_2})$.
  So for $\alpha_{t_m} \leq T_{m}^2$ to hold for sufficiently large $m$,
  it is sufficient to have
    $\frac{\nu(1-{\beta_2})}{1 - \nu} > 2{\beta_2}$,
  which is equivalent to ${\beta_2} < \frac{\nu}{2 - \nu}$ in~\eqref{eq def eta}.
  
  \tb{Case 2: $\nu = 1$ in Assumption~\ref{assu lr}.} In this case, we have $T_m = \Theta\qty(\frac{\ln^{\beta_1}m}{m})$.
  A similar integral approximation yields
    $\sum_{i=0}^m T_i = \Theta\qty(\ln^{\beta_1+1}(m+1))$ and $\sum_{t=0}^{t_{m+1} - 1} \alpha_t = \Theta\qty(\ln t_{m+1})$.
  Then we have $\ln^{\beta_1+1}(m) \leq C_{\tref{lem lr bounds},1}\ln t_{m}$ for some constant $C_{\tref{lem lr bounds}, 1}$ and sufficiently large $m$.
  It follows that
  \begin{align}
    \textstyle\alpha_{t_{m}} = \fO\qty(t_m^{-1}) = \fO\qty(\qty(\exp(\qty(\ln^{\beta_1 + 1} m) / C_{\tref{lem lr bounds},1}))^{-1}) = \fO\qty({m^{-\qty(\ln^{\beta_1}m )/ C_{\tref{lem lr bounds}, 1}}}).
  \end{align}
  For $\alpha_{t_m} \leq T_{m}^2$ to hold for sufficiently large $m$,
  it is sufficient to have 
  \begin{align}
    \textstyle\frac{1}{m^{(\ln^{\beta_1}m) / C_{\tref{lem lr bounds}, 1}}} \leq \frac{\ln^{2\beta_1}m}{m^2}.
  \end{align}
  This is true because $(\ln^{\beta_1}m)/C_{\tref{lem lr bounds},1} > 2$ holds for sufficiently large $m$.

  \tb{Case 3: Assumption~\ref{assu lr2}.} In this case, we have $T_m = \Theta(\frac{1}{m})$.
  Using integral approximation, we get $\sum_{i=0}^m T_i = \Theta(\ln (m+1))$ and $\sum_{t=0}^{t_{m+1} - 1} \alpha_t = \Theta(\ln^{1-\omega} t_{m+1})$.
  Then we have $\ln m \leq C_{\tref{lem lr bounds},2} \ln^{1-\omega} t_m$ for some constant $C_{\tref{lem lr bounds}, 2}$ and sufficiently large $m$. It follows that
  \begin{align}
    \textstyle\alpha_{t_m} =&\textstyle \fO\qty(\frac{1}{t_m \ln^\omega t_m}) = \fO\qty( \qty(\exp(\qty(\frac{\ln m}{C_{\tref{lem lr bounds},2}})^{\frac{1}{1-\omega}}) \qty(\frac{\ln m}{C_{\tref{lem lr bounds},2}})^{\frac{\omega}{1-\omega}})^{-1} ) \\
    =&\textstyle \fO\qty( \qty(m^{\qty(\ln^{\omega/(1-\omega)}m)  C_{\tref{lem lr bounds},2}^{-1/(1-\omega)}} \qty(\frac{\ln m}{C_{\tref{lem lr bounds},2}})^{\frac{\omega}{1-\omega}})^{-1}).
  \end{align}
  Since $\qty(\ln^{\omega/(1-\omega)}m) C_{\tref{lem lr bounds}, 2}^{-1/(1-\omega)} > 2$ holds for sufficiently large $m$,
  we conclude that $\alpha_{t_m} \leq T_m^2$ holds for sufficiently large $m$,
  which completes the proof.
\end{proof}

\subsection{Proof of Lemma \ref{lem lr bounds 2}}
\label{sec proof lem lr bounds 2}
\begin{proof}
  By~\eqref{eq def tm}, we must have $t_{m+1} \geq t_m + 1$.
  Let $m$ be sufficiently large such that $C_{\tref{lem lr bounds}} T_m < 1$.
  If $t_{m+1} = t_m + 1$,
  we have
    $\bar \alpha_m = \alpha_{t_m} \leq C_\tref{lem lr bounds}T_m^2 < 2T_m$.
  If $t_{m+1} \geq t_m + 2$,
  we have by the $\min$ operator in~\eqref{eq def tm} that
    $\bar \alpha_m - \alpha_{t_{m+1} - 2} < T_m$.
  This means that $\bar \alpha_m \leq T_m + \alpha_{t_{m+1} - 2} \leq T_m + \alpha_{t_m} \leq T_m + C_\tref{lem lr bounds}T_m^2 \leq 2T_m$.
  This completes the proof.
\end{proof}

\subsection{Proof of Lemma~\ref{lem bound leading term}}
\label{sec proof lem bound leading term}
\begin{proof}
  \begin{align}
    &\textstyle \frac{1}{2} \indot{\nabla \norm{w_{t_m} - w_*}_M^2}{g(w_{t_m})} \\
    =&\textstyle  \frac{1}{2}\indot{\nabla \norm{w_{t_m} - w_*}_M^2}{h(w_{t_m}) - w_*} + \frac{1}{2}\indot{\nabla \norm{w_{t_m} - w_*}_M^2}{w_* - w_{t_m}} \\
    \leq&\textstyle  \norm{w_{t_m} - w_*}_M \norm{h(w_{t_m}) - w_*}_M + \frac{1}{2}\indot{\nabla \norm{w_{t_m} - w_*}_M^2}{w_* - w_{t_m}} \explain{Lemma~\ref{lem property of M}}\\
    =&\textstyle \norm{w_{t_m} - w_*}_M \norm{h(w_{t_m}) - w_*}_M  - \frac{1}{2}\indot{\nabla \norm{w_{t_m} - w_*}_M^2}{w_{t_m} - w_*} \\
    \leq&\textstyle  \norm{w_{t_m} - w_*}_M \norm{h(w_{t_m}) - w_*}_M  - \norm{w_{t_m} - w_*}_M^2 \explain{Lemma~\ref{lem property of M}} \\
    \leq&\textstyle  \norm{w_{t_m} - w_*}_M \frac{u_{cm}}{l_{cm}}\kappa\norm{w_{t_m} - w_*}_M  - \norm{w_{t_m} - w_*}_M^2 \explain{Lemma~\ref{lem property of M}} \\
    =&\textstyle -(1 - \frac{u_{cm}}{l_{cm}}\kappa) \norm{w_{t_m} - w_*}_M^2,
  \end{align}
  which completes the proof.
\end{proof}

\subsection{Proof of Lemma~\ref{lem bound z1}}
\label{sec proof lem bound z1}
\begin{proof}
\begin{align}
  \textstyle\norm{z_{1, m}} \leq&\textstyle \sum_{t=t_m}^{t_{m+1} - 1} \alpha_t (L_h + 1) \norm{w_t - w_{t_m}} \\
  \leq&\textstyle \sum_{t=t_m}^{t_{m+1} - 1} \alpha_t (L_h + 1) \bar \alpha_m C_\text{\ref{lem bound gronwall}}\qty(\norm{w_{t_m} - w_*} + 1) \explain{Lemma~\ref{lem bound gronwall}} \\
  =&\textstyle\bar \alpha_m^2  (L_h + 1) C_\text{\ref{lem bound gronwall}}\qty(\norm{w_{t_m} - w_*} + 1) \\
  \leq &\textstyle 4 T_m^2  (L_h + 1) C_\text{\ref{lem bound gronwall}}\qty(\norm{w_{t_m} - w_*} + 1) \explain{Lemma~\ref{lem lr bounds 2}}.
\end{align}
The equivalence between norms then completes the proof.
\end{proof}
\subsection{Proof of Lemma~\ref{lem bound z2}}
\label{sec proof lem bound z2}
\begin{proof}
  \begin{align}
    \textstyle\norm{z_{2, m}} \leq&\textstyle \sum_{t=t_m}^{t_{m+1}-1} \alpha_t \norm{G(w_{t_m}, Y_{t+1}) - \E_m\qty[G(w_{t_m}, Y_{t+1})]} \\
    \leq&\textstyle \sum_{t=t_m}^{t_{m+1}-1} \alpha_t 2 \sup_y \norm{G(w_{t_m}, y)} \\
    \leq&\textstyle \sum_{t=t_m}^{t_{m+1}-1} \alpha_t 2 C_\tref{lem bound G} \qty(\norm{w_{t_m}} + 1) \explain{Lemma~\ref{lem bound G}} \\
    \leq&\textstyle \bar \alpha_m 2 C_\tref{lem bound G} \qty(\norm{w_{t_m} - w_*} + \norm{w_*} + 1) \\
    \leq&\textstyle 2 T_m 2 C_\tref{lem bound G} \qty(\norm{w_{t_m} - w_*} + \norm{w_*} + 1) \explain{Lemma~\ref{lem lr bounds 2}}
  \end{align}
  The equivalence between norms then completes the proof.
\end{proof}
\subsection{Proof of Lemma~\ref{lem bound z3}}
\label{sec proof lem bound z3}
\begin{proof}
  \begin{align}
    \textstyle\norm{z_{3, m}} \leq&\textstyle \sum_{t=t_m}^{t_{m+1}-1} \alpha_t \norm{\E_m\qty[G(w_{t_m}, Y_{t+1}) - g(w_{t_m})]} \\
    \leq&\textstyle \sum_{t=t_m}^{t_{m+1}-1} \alpha_t \norm{\int_{y\in \fY} \qty(P^{t+1-t_m}(Y_{t_m}, y) - d_\fY(y)) G(w_{t_m}, y) \dd y} \\
    \leq&\textstyle \sum_{t=t_m}^{t_{m+1}-1} \alpha_t \int_{y\in \fY} \abs{\qty(P^{t+1-t_m}(Y_{t_m}, y) - d_\fY(y))} \norm{G(w_{t_m}, y)} \dd y \\
    \leq&\textstyle \sum_{t=t_m}^{t_{m+1}-1} \alpha_t \sup_y \norm{G(w_{t_m}, y)} \int_{y\in \fY} \abs{\qty(P^{t+1-t_m}(Y_{t_m}, y) - d_\fY(y))} \dd y \\
    \leq&\textstyle \sum_{t=t_m}^{t_{m+1}-1} \alpha_t C_\tref{lem bound G} \qty(\norm{w_{t_m} - w_*} + \norm{w_*} + 1) C_\aref{assu markov chain}\varrho^{t+1-t_m} \\
    \leq&\textstyle \sum_{t=t_m}^{t_{m+1}-1} C_\tref{lem lr bounds} T_m^2 C_\tref{lem bound G} \qty(\norm{w_{t_m} - w_*} + \norm{w_*} + 1) C_\aref{assu markov chain}\varrho^{t +1 - t_m} \\
    \leq&\textstyle  \frac{T_m^2 C_\tref{lem lr bounds} C_\tref{lem bound G} \qty(\norm{w_{t_m} - w_*} + \norm{w_*} + 1) C_\aref{assu markov chain} \varrho}{1 - \varrho}.
  \end{align}
  The equivalence between norms then completes the proof.
\end{proof}
\subsection{Proof of Lemma~\ref{lem bound all}}
\label{sec proof lem bound all}
\begin{proof}
  We first bound the last two terms in~\eqref{eq main smooth ineq}.
  For the first of the two, we have
  \begin{align}
    &\textstyle \indot{\nabla \norm{w_{t_m} - w_*}_M^2}{z_{1, m} + z_{3, m}} \\
    \leq&\textstyle 2\norm{w_{t_m} - w_*}_M \norm{z_{1, m} + z_{3, m}}_M \explain{Lemma~\ref{lem property of M}} \\
    \leq&\textstyle 2\norm{w_{t_m} - w_*}_M (\norm{z_{1, m}}_M + \norm{z_{3, m}}_M) \\
    \leq&\textstyle 2\norm{w_{t_m} - w_*}_M T_m^2 (C_\tref{lem bound z1} + C_\tref{lem bound z3}) (\norm{w_{t_m} - w_*}_M + 1) \\
    \leq&\textstyle 2T_m^2 (C_\tref{lem bound z1} + C_\tref{lem bound z3}) (\norm{w_{t_m} - w_*}_M^2 + \norm{w_{t_m} - w_*}_M) \\
    \leq&\textstyle 2T_m^2 (C_\tref{lem bound z1} + C_\tref{lem bound z3}) (\norm{w_{t_m} - w_*}_M^2 + \frac{1}{2}(\norm{w_{t_m} - w_*}_M^2+1)) \\
    \leq&\textstyle 3T_m^2 (C_\tref{lem bound z1} + C_\tref{lem bound z3}) (\norm{w_{t_m} - w_*}_M^2 + 1).
  \end{align}
  For the second of the two,
  we notice from Lemma~\ref{lem bound G} that
  \begin{align}
    \textstyle\norm{g(w_{t_m})}_M \leq C_\tref{lem bound G}(\norm{w_{t_m}}_M + 1).
  \end{align}
  With the fact that $\bar \alpha_m \leq 2T_m$, it then follows easily that 
  \begin{align}
    \textstyle\norm{\bar \alpha_m g(w_{t_m})}_M
    \leq 2C_\tref{lem bound G}T_m (\norm{w_{t_m} - w_*}_M + \norm{w_*}_M +1)
  \end{align}
  The equivalence between norms then confirms that 
  \begin{align}
    &\textstyle\norm{\bar \alpha_m g(w_{t_m}) + z_{1, m} + z_{2, m} + z_{3, m}}_s^2 \\
    \leq&\textstyle 4(\norm{\bar \alpha_m g(w_{t_m})}_s^2+\norm{z_{1, m}}_s^2+\norm{z_{2, m}}_s^2+\norm{z_{3, m}}_s^2) \\
    \leq&\textstyle 4\left(u_{cm}/l_{cs}\right)^2(\norm{\bar \alpha_m g(w_{t_m})}_M^2+\norm{z_{1, m}}_M^2+\norm{z_{2, m}}_M^2+\norm{z_{3, m}}_M^2) \\
    \leq&\textstyle 4\left(u_{cm}/l_{cs}\right)^2[4T_m^2C_\tref{lem bound G}^2(\norm{w_{t_m} - w_*}_M + \norm{w_*}_M + 1)^2 \\
  &\textstyle\quad + 2T_m^2(1 + T_m^2)(C_\tref{lem bound z1}^2 + C_\tref{lem bound z2}^2 + C_\tref{lem bound z3}^2)(\norm{w_{t_m} - w_*}_M^2 + 1)] \\
    \leq&\textstyle C_{\tref{lem bound all},1}T_m^2(\norm{w_{t_m} - w_*}^2_M + 1),
  \end{align}
  where we may take
  \begin{align*}
  C_{\tref{lem bound all},1}
  \doteq 4\left(\frac{u_{cm}}{l_{cs}}\right)^2
  \bigl[&8C_\tref{lem bound G}^2(\norm{w_*}_M+1)^2\\
  &+2(1+T_0^2)(C_\tref{lem bound z1}^2+C_\tref{lem bound z2}^2+C_\tref{lem bound z3}^2)\bigr].
  \end{align*}
  Combining the two remainder terms in~\eqref{eq main smooth ineq}, set
  \begin{align}
  C_\tref{lem bound all}\doteq
  3(C_\tref{lem bound z1}+C_\tref{lem bound z3})
  +\frac{L}{\xi}C_{\tref{lem bound all},1}.
  \end{align}
  We then obtain
  \begin{align}
    &\textstyle\norm{w_{t_{m+1}} - w_*}_M^2 \\
    \leq&\textstyle (1 - 2\bar \alpha_m \kappa')\norm{w_{t_m} - w_*}_M^2 + \indot{\nabla \norm{w_{t_m} - w_*}_M^2}{z_{2,m}} + C_\tref{lem bound all}T_m^2(\norm{w_{t_m} - w_*}^2_M + 1)  \\
    \leq&\textstyle (1 - 2 T_m \kappa')\norm{w_{t_m} - w_*}_M^2 + \indot{\nabla \norm{w_{t_m} - w_*}_M^2}{z_{2,m}} + C_\tref{lem bound all}T_m^2(\norm{w_{t_m} - w_*}^2_M + 1).
  \end{align}
  Let $m_0 \doteq \min\{m: T_m \leq \min\{\kappa'/C_\tref{lem bound all}, 1\}\}$. Then for all $m \geq m_0$, we have
  \begin{align}
    \textstyle\norm{w_{t_{m+1}} - w_*}_M^2 \leq (1 - T_m \kappa')\norm{w_{t_m} - w_*}_M^2 + \indot{\nabla \norm{w_{t_m} - w_*}_M^2}{z_{2,m}} + C_\tref{lem bound all}T_m^2,
  \end{align}
  which completes the proof.
  
\end{proof}

\subsection{Proof of Lemma~\ref{lem as rate skeleton sa}}
\label{sec proof lem as rate skeleton sa}
\begin{proof}
We recall that $\E_m\qty[z_{2, m}] = 0$.
Then, taking conditional expectations on both sides of Lemma~\ref{lem bound all}
yields that for $m \geq m_0$, 
\begin{align}
\textstyle\E_m\qty[\norm{w_{t_{m+1}} - w_*}_M^2] \leq (1 - T_m\kappa')\norm{w_{t_m} - w_*}_M^2 + C_\tref{lem bound all}T_m^2.
\end{align}
We now proceed via applying Lemma~\ref{lem almost smg}.
In particular,
we identify
\begin{align}
    \textstyle Z_m \textstyle\leftrightarrow \e{m+m_0}, 
     A_m \textstyle\leftrightarrow 0, 
     B_m \textstyle\leftrightarrow C_\tref{lem bound all}T_{m+m_0}^2,
     \alpha_m \textstyle\leftrightarrow T_{m+m_0}\kappa'.
\end{align}
Then Condition (i) holds.
The definition of $T_m$ in~\eqref{eq def big tm} trivially confirms that Condition (ii) holds.
Let the $\epsilon$ in Lemma~\ref{lem almost smg} be any number in $(0, 2\beta_2 - 1)$.

When $\nu < 1$, for Condition (iii), 
we have $T_m = \frac{C_\alpha}{(m+3)^{\beta_2}}$
with $\beta_2 < 1$.
So when $m$ is sufficiently large,
it holds that $T_{m+m_0} \kappa' \geq \frac{\epsilon}{m}$.
Condition (iii) holds.
For Condition (iv),
we have from $\epsilon < 2\beta_2 - 1$ that
\begin{align}
\textstyle\sum_{m=1}^\infty (m+1)^\epsilon B_m = C_\alpha^2C_\tref{lem bound all}\sum_{m=1}^\infty \frac{(m+1)^\epsilon}{(m + m_0+3)^{2\beta_2}} < \infty,
\end{align}
and from $\beta_2 < 1$ that
\begin{align}
\textstyle\sum_{m=1}^\infty (\alpha_m - \frac{\epsilon}{m}) = \sum_{m=1}^\infty \qty(\frac{C_\alpha\kappa'}{(m+m_0+3)^{\beta_2}} - \frac{\epsilon}{m}) = \infty.
\end{align}
So Condition (iv) holds.

When $\nu=1$,
we have $\beta_2 = 1$ and $\epsilon \in (0, 1)$.
For Condition (iii),
we have $T_m = \frac{C_\alpha \ln^{\beta_1} (m+3)}{m+3}$,
implying that Condition (iii) holds.
For Condition (iv),
we have from $\epsilon < 1$ that
\begin{align}
\textstyle\sum_{m=1}^\infty (m+1)^\epsilon B_m = C_\alpha^2C_\tref{lem bound all}\sum_{m=1}^\infty \frac{(m+1)^\epsilon \ln^{2\beta_1} (m+m_0+3)}{(m + m_0+3)^2} < \infty,
\end{align}
and
\begin{align}
\textstyle\sum_{m=1}^\infty (\alpha_m - \frac{\epsilon}{m}) = \sum_{m=1}^\infty \qty(\frac{C_\alpha \ln^{\beta_1}(m+m_0+3)\kappa'}{m+m_0+3} - \frac{\epsilon}{m}) = \infty.
\end{align}
So Condition (iv) holds.
Having verified all conditions of Lemma~\ref{lem almost smg} for all setups in Assumption~\ref{assu lr},
invoking Lemma~\ref{lem almost smg} and using $\lim_{m\to\infty} \frac{(m+m_0)^\epsilon}{m^\epsilon} = 1$ then completes the proof. 
\end{proof}

\subsection{Proof of Theorem~\ref{thm lil}}
\label{sec proof lem connect skeleton sa and sa}
\begin{proof}
  Having established the convergence rate of~\eqref{eq skeleton sa} in Lemma~\ref{lem as rate skeleton sa},
  we now complete the proof by mapping the convergence rate back to~\eqref{eq sa update}.
  We recall that we proceed under Assumption~\ref{assu lr}
  and recall the $m_0$ in Lemma~\ref{lem lr bounds 2}.
  Then for all $m \geq m_0$, it holds that
  \begin{align}
    \label{eq lem connect big theta}
    \textstyle\sum_{i=m_0}^m T_i \leq \sum_{i=m_0}^{m} \bar \alpha_i = \sum_{t=t_{m_0}}^{t_{m+1}-1} \alpha_t \leq \sum_{i=m_0}^m \bigl(T_i + C_\tref{lem lr bounds} T_i^2\bigr) \leq 2\sum_{i=m_0}^m T_i.
  \end{align}
  From Lemma~\ref{lem bound gronwall}, 
  it is easy to see for any $t \in [t_m, t_{m+1}]$,
  \begin{align}
    \label{eq connection basic relation}
    \textstyle\norm{w_t - w_*}^2 \leq& 2C_{\tref{lem bound gronwall}}^2(\bar \alpha_m^2 + \norm{w_{t_m} - w_*}^2) \leq 8C_{\tref{lem bound gronwall}}^2(T_m^2 + \norm{w_{t_m} - w_*}^2).
  \end{align}
  For any $\epsilon \in (0, 2\beta_2 - 1)$, we have
  \begin{align}
    \textstyle m^\epsilon \sup_{t\in [t_m, t_{m+1}]} \norm{w_t - w_*}^2 \leq 8C_{\tref{lem bound gronwall}}^2(m^\epsilon T_m^2 + m^\epsilon \norm{w_{t_m} - w_*}^2).
  \end{align}
  The definition of $T_m$ in~\eqref{eq def big tm} and Lemma~\ref{lem as rate skeleton sa} then confirm that
  \begin{align}
    \textstyle\lim_{m\to\infty} m^\epsilon \sup_{t\in [t_m, t_{m+1}]} \norm{w_t - w_*}^2=0.
  \end{align}
  For $\nu < 1$,
  using~\eqref{eq lem connect big theta} and integral approximation (cf. Section~\ref{sec proof lem lr bounds}),
  it is easy to get $m^{1-\beta_2} = \Theta(t_m^{1-\nu})$.
  We then have
  \begin{align}
    \textstyle0 =&\textstyle \lim_{m\to\infty} (m-1)^\epsilon \sup_{t\in [t_{m-1}, t_m]} \norm{w_t - w_*}^2 = \lim_{m\to\infty} m^\epsilon \sup_{t\in [t_{m-1}, t_m]} \norm{w_t - w_*}^2 \\
    =&\textstyle\lim_{m\to\infty} t_m^{\frac{1-\nu}{1-\beta_2} \epsilon} \sup_{t\in [t_{m-1}, t_m]} \norm{w_t - w_*}^2 \\
    \geq&\textstyle\lim_{m\to\infty} \sup_{t\in [t_{m-1}, t_m]} t^{\frac{1-\nu}{1-\beta_2} \epsilon} \norm{w_t - w_*}^2 = \lim_{t\to\infty} t^{\frac{1-\nu}{1-\beta_2} \epsilon} \norm{w_t - w_*}^2.
  \end{align} 
  We now optimize the choice of ${\beta_2} \in (1/2, \nu/(2-\nu))$ and $\epsilon \in (0, 2{\beta_2} - 1)$.
  Notice that
    $\sup_{{\beta_2} \in (1/2, \nu/(2-\nu)), \epsilon \in (0, 2{\beta_2} - 1)} {\frac{1-\nu}{1 - {\beta_2}}\epsilon} = \frac{3}{2}\nu - 1$.
  We then conclude that for any $\zeta < \frac{3}{2}\nu - 1$,
  \begin{align}
    \textstyle\lim_{t\to\infty} t^\zeta \norm{w_t - w_*}^2 = 0.
  \end{align}
  For $\nu=1$, using~\eqref{eq lem connect big theta} and integral approximation (cf. Section~\ref{sec proof lem lr bounds}),
  we get
  \begin{align}
      &\textstyle C_\alpha \ln^{1+\beta_1}(m+3) + C_\tref{lem lr bounds}C_\alpha^2 \sum_{i=m_0}^\infty \frac{\ln^{2\beta_1}(i+3)}{(i+3)^2} \geq C_\alpha\qty(\ln (t_{m+1} + 2) - \ln(t_{m_0} + 3)), \\
      &\textstyle\ln(m+3) \geq \qty(\ln (t_{m+1} + 2) - \ln(t_{m_0} + 3) - C_\tref{lem lr bounds}C_\alpha \sum_{i=m_0}^\infty \frac{\ln^{2\beta_1}(i+3)}{(i+3)^2})^{\frac{1}{1+\beta_1}}.
  \end{align}
  To simplify this, we should clarify an inequality first.
  \begin{lemma}
      For any $\nu \in (0,1)$ and $a, b>0$, we have \(a^\nu + b^\nu \geq (a+b)^\nu\).
  \end{lemma} 
  \begin{proof}
      First, notice that to conclude, we only need to prove $f(x)=1+x^\nu - (1+x)^\nu >0, \forall x>0$. 
      We have
          $f'(x)=\nu x^{\nu -1} - \nu (1+x)^{\nu -1}$.
      It then follows from $\nu - 1 < 0$ that $f'(x) > 0$ for every $x>0$.
      Since $f$ is continuous on $[0,x]$ and differentiable on $(0,x)$, the mean value theorem gives $f(x)>f(0)=0$ for every $x>0$.
      Taking $x=b/a$ and multiplying by $a^\nu$ proves the stated inequality.
  \end{proof}
  We now identify $b = \ln(t_{m_0} + 3) + C_\tref{lem lr bounds}C_\alpha \sum_{i=m_0}^\infty \frac{\ln^{2\beta_1}(i+3)}{(i+3)^2}$ and $a= \ln(t_{m+1} + 2) - b$.
  When $m$ is sufficiently large, $a$ is positive.
  So we get
  \begin{align}
      \textstyle \ln(m+3) \geq \qty(\ln (t_{m+1} + 2))^{\frac{1}{1+\beta_1}} - \qty(\ln(t_{m_0} + 3) + C_\tref{lem lr bounds}C_\alpha \sum_{i=m_0}^\infty \frac{\ln^{2\beta_1}(i+3)}{(i+3)^2})^{\frac{1}{1+\beta_1}}.
  \end{align}
Define $C_{Th\tref{thm lil}} \doteq \exp(-\zeta \qty(\ln(t_{m_0} + 3) + C_\tref{lem lr bounds}C_\alpha \sum_{i=m_0}^\infty \frac{\ln^{2\beta_1}(i+3)}{(i+3)^2})^{\frac{1}{1+\beta_1}})$,
  then we have
  \begin{align}
    \textstyle0 =&\textstyle \lim_{m\to\infty} (m+2)^\zeta \sup_{t\in [t_{m-1}, t_m]} \norm{w_t - w_*}^2 \\
    \geq&\textstyle \lim_{m\to\infty} C_{{Th}\tref{thm lil}}\exp(\zeta \ln^{1/(1+\beta_1)}t_m) \sup_{t\in [t_{m-1}, t_m]} \norm{w_t - w_*}^2 \\
    \geq&\textstyle C_{{Th}\tref{thm lil}} \lim_{t\to\infty} \exp(\zeta \ln^{1/(1+\beta_1)}t) \norm{w_t - w_*}^2,
  \end{align}
  This proves the claim for $0<\beta_1<1$.
  For $\beta_1\geq1$, fix any $\tilde{\beta}_1\in(0,1)$.
  Since $1/(1+\beta_1)<1/(1+\tilde{\beta}_1)$, for $t\geq e$ the normalization
  $\exp(\zeta\ln^{1/(1+\beta_1)}t)$ is no larger than
  $\exp(\zeta\ln^{1/(1+\tilde{\beta}_1)}t)$, so the result with $\tilde{\beta}_1$ yields the claimed almost sure limit.

\end{proof}

\subsection{Proof of Lemma~\ref{lem inductive basis}}
\label{proof of lem inductive basis}
\begin{proof}
    For $m\ge m_0$, where $m_0$ is a sufficiently large constant, we have
    \begin{align}
        \textstyle\norm{w_{t_{m+1}} - w_*}_M \leq &\textstyle\norm{w_{t_m} - w_*}_M + \bar\alpha_m \norm{g(w_{t_m})}_M + \norm{z_m}_M \\
        \leq &\textstyle\norm{w_{t_m} - w_*}_M + \bar\alpha_m C_\tref{lem bound G} (\norm{w_{t_m} - w_*}_M + \norm{w_*}_M + 1) \\
        &\textstyle + (T_m^2 C_\tref{lem bound z1} + T_m C_\tref{lem bound z2} + T_m^2 C_\tref{lem bound z3})\left( \norm{w_{t_m} - w_*}_M +1\right) \\
        \leq &\textstyle(1 + 2T_m C_\tref{lem bound G} + T_m^2 C_\tref{lem bound z1} + T_m C_\tref{lem bound z2} + T_m^2 C_\tref{lem bound z3})\norm{w_{t_m} - w_*}_M \\
        &\textstyle + (2T_m C_\tref{lem bound G} + T_m^2 C_\tref{lem bound z1} + T_m C_\tref{lem bound z2} + T_m^2 C_\tref{lem bound z3})(\norm{w_*}_M + 1).
    \end{align}
    Denote $C_{\tref{lem inductive basis},1} \doteq 2C_\tref{lem bound G} +  T_0 C_\tref{lem bound z1} + C_\tref{lem bound z2} + T_0 C_\tref{lem bound z3}$, 
    $C_{\tref{lem inductive basis}, 2} \doteq (2T_0 C_\tref{lem bound G} + T_0^2 C_\tref{lem bound z1} + T_0 C_\tref{lem bound z2} + T_0^2 C_\tref{lem bound z3})(\norm{w_*}_M + 1)$. 
    Then we have
    \begin{align}
        \textstyle\norm{w_{t_{m+1}} - w_*}_M \leq &\textstyle\norm{w_{t_m} - w_*}_M + C_{\tref{lem inductive basis},1}T_m \norm{w_{t_m} - w_*}_M + C_{\tref{lem inductive basis}, 2} \\
        \leq &\textstyle\norm{w_{t_{m_0}} - w_*}_M + \sum_{i=m_0}^m C_{\tref{lem inductive basis},1}T_i \norm{w_{t_i} - w_*}_M + \sum_{i=m_0}^m C_{\tref{lem inductive basis}, 2}.
    \end{align}
    From~\eqref{eq ttm gronwall 3}, 
    we have for any $t \leq t_{m_0}$ (including $t_m$ with $m \leq m_0$),
    \begin{align}
      \norm{w_t - w_*} \leq C_{\tref{lem inductive basis}, 3} \doteq \textstyle(\norm{w_0 - w_*} + C_\tref{lem bound gronwall}\sum_{\tau=0}^{t_{m_0}-1} \alpha_\tau ) \exp(C_\tref{lem bound gronwall} \sum_{\tau=0}^{t_{m_0}-1} \alpha_\tau).
    \end{align}
    Applying the discrete Gronwall inequality then yields that for all $m\geq m_0$,
    \begin{align}
      \textstyle\norm{w_{t_{m}} - w_*}_M \leq \qty(\frac{u_{cm}}{l_{cm}}C_{\tref{lem inductive basis}, 3} + (m -m_0+1) C_{\tref{lem inductive basis}, 2}) \exp\qty(\sum_{i=m_0}^{m-1} C_{\tref{lem inductive basis},1}T_i).
    \end{align}
    Since $T_i=C_\alpha/(i+3)$, we have
    $\sum_{i=m_0}^{m-1}T_i\leq C_\alpha\ln(m+3)$.
    Squaring the preceding bound gives a constant times
    $(m+3)^{2+2C_{\tref{lem inductive basis},1}C_\alpha}$.
    Choose an integer $C_\tref{lem inductive basis}\geq2+2C_{\tref{lem inductive basis},1}C_\alpha$.
    Taking the integer $C_\tref{lem inductive basis}'\geq1$ sufficiently large absorbs this constant and, by the bound for $t\leq t_{m_0}$ above and norm equivalence, covers all $0\leq m<m_0$ as well.
\end{proof}

\subsection{Proof of Lemma~\ref{lem inductive step}}
\label{proof of lem inductive step}
\begin{proof}
For any $\delta \in [0, 1)$, consider the deterministic, non-decreasing sequence $\qty{B_m(\delta)}$ in~\eqref{eq prob ineq 1}, with $B_m(\delta)\geq1$.
Enlarge $m_0$ so that $T_m\kappa'\leq1$ for every $m\geq m_0$, and take $C_{\tref{lem bound z2}}>0$, increasing that bound if necessary.
For each $m \geq m_0$,
define an event $E_{m}(\delta) = \qty{\norm{w_{t_k} - w_*}_M^2 \leq B_k(\delta), \, \forall k = m_0, m_0 + 1, \dots, m}$. 
Notably, $\qty{E_m(\delta)}_{m \geq m_0}$ is by definition a sequence of decreasing events, i.e., $E_{m+1}(\delta) \subset E_m(\delta)$.
In addition, according to~\eqref{eq prob ineq 1}, 
we have $\Pr(E_m(\delta)) \geq 1 - \delta$ for any $m \geq m_0$. 
Let $\lambda_m = \theta T_m^{-1} B_m(\delta)^{-1}$
where $\theta$ is a tunable constant yet to be chosen.

We now proceed to show that $\qty{\exp \left(\lambda_m \mathbbm{1}_{E_m(\delta)} \norm{w_{t_m} - w_*}_M^2 \right)}$ is ``almost'' a supermartingale.
First, by first multiplying $\lambda_{m+1} \mathbbm{1}_{E_{m+1}(\delta)}$ and then taking exponential on both sides of Lemma~\ref{lem bound all}, 
\begin{align}
    &\textstyle\exp(\lambda_{m+1} \mathbbm{1}_{E_{m+1}(\delta)}\norm{w_{t_{m+1}} - w_*}_M^2)\\
    \leq&\textstyle \exp(\lambda_{m+1} \mathbbm{1}_{E_{m+1}(\delta)}(1 - T_m \kappa')\norm{w_{t_m} - w_*}_M^2)\\
    &\textstyle\times \exp(\lambda_{m+1} \mathbbm{1}_{E_{m+1}(\delta)}\indot{\nabla \norm{w_{t_m} - w_*}_M^2}{z_{2,m}}) \times \exp(\lambda_{m+1} \mathbbm{1}_{E_{m+1}(\delta)}C_\tref{lem bound all} T_m^2)\\
    \leq&\textstyle \exp(\lambda_{m+1} \mathbbm{1}_{E_{m}(\delta)}(1 - T_m \kappa')\norm{w_{t_m} - w_*}_M^2)\\
    &\textstyle\times \exp(\lambda_{m+1} \mathbbm{1}_{E_{m}(\delta)}\indot{\nabla \norm{w_{t_m} - w_*}_M^2}{z_{2,m}}) \times \exp(\lambda_{m+1} \mathbbm{1}_{E_{m}(\delta)}C_\tref{lem bound all} T_m^2),
\end{align}
For the last inequality, $E_{m+1}(\delta)\subseteq E_m(\delta)$ and the full right-hand side of the squared-error bound in Lemma~\ref{lem bound all} is non-negative, since it bounds $\norm{w_{t_{m+1}}-w_*}_M^2$ from above.
Taking expectation conditioned on $\mathcal{F}_{t_m}$ on both sides of the previous inequality, we obtain
\begin{align}
    &\textstyle\mathbb{E}\qty[\exp(\lambda_{m+1} \mathbbm{1}_{E_{m+1}(\delta)}\norm{w_{t_{m+1}} - w_*}_M^2) | \mathcal{F}_{t_m}]\\
    \leq&\textstyle \exp(\lambda_{m+1} \mathbbm{1}_{E_{m}(\delta)}(1 - T_m \kappa')\norm{w_{t_m} - w_*}_M^2) \times \exp(\lambda_{m+1} \mathbbm{1}_{E_{m}(\delta)}C_\tref{lem bound all} T_m^2)\\
    &\textstyle\times \mathbb{E}\left[\exp(\lambda_{m+1} \mathbbm{1}_{E_{m}(\delta)}\indot{\nabla \norm{w_{t_m} - w_*}_M^2}{z_{2,m}}) | \mathcal{F}_{t_m}\right]\label{eq bound exp}.
\end{align}
To bound the last term, we will use Hoeffding's lemma. 
We recall $\E_m\qty[z_{2, m}] = 0$. 
Additionally, we have
\begin{align}
    &\textstyle\abs{\lambda_{m+1} \mathbbm{1}_{E_{m}(\delta)}\indot{\nabla \norm{w_{t_m} - w_*}_M^2}{z_{2,m}}}\\
    \leq&\textstyle 2\lambda_{m+1} \mathbbm{1}_{E_{m}(\delta)}\norm{w_{t_m} - w_*}_M\norm{z_{2,m}}_M \explain{Lemma~\ref{lem property of M}}\\
    \leq&\textstyle 2\lambda_{m+1} \mathbbm{1}_{E_{m}(\delta)}\norm{w_{t_m} - w_*}_M T_m C_{\tref{lem bound z2}} (\norm{w_{t_m} - w_*}_M + 1) \explain{Lemma~\ref{lem bound z2}}\\
    \leq&\textstyle 2\lambda_{m+1} \mathbbm{1}_{E_{m}(\delta)}\norm{w_{t_m} - w_*}_M T_m C_{\tref{lem bound z2}} (B_m(\delta)^{1/2} + 1) \explain{Definition of $E_m(\delta)$}.
\end{align}
Now we can apply Hoeffding's lemma to the last term in \eqref{eq bound exp}, which gives us
\begin{align}
    &\textstyle\mathbb{E}\left[\exp(\lambda_{m+1} \mathbbm{1}_{E_{m}(\delta)}\indot{\nabla \norm{w_{t_m} - w_*}_M^2}{z_{2,m}}) | \mathcal{F}_{t_m}\right]\\
    \leq&\textstyle \exp(16\lambda_{m+1}^2 \mathbbm{1}_{E_{m}(\delta)}\norm{w_{t_m} - w_*}^2_M T^2_m C^2_{\tref{lem bound z2}} (B_m(\delta)^{1/2} + 1)^2 / 8) \\
    \leq&\textstyle \exp(4\lambda_{m+1}^2 \mathbbm{1}_{E_{m}(\delta)}\norm{w_{t_m} - w_*}^2_M T^2_m C^2_{\tref{lem bound z2}} (B_m(\delta) + 1)).
\end{align}
Applying this back to \eqref{eq bound exp}, we have 
\begin{align}
    &\textstyle\mathbb{E}\qty[\exp(\lambda_{m+1} \mathbbm{1}_{E_{m+1}(\delta)}\norm{w_{t_{m+1}} - w_*}_M^2) | \mathcal{F}_{t_m}]\\
    \leq&\textstyle \exp(\lambda_{m+1} \mathbbm{1}_{E_{m}(\delta)}\norm{w_{t_m} - w_*}_M^2(1 - T_m \kappa'+ 4\lambda_{m+1}T_m^2C_{\tref{lem bound z2}}^2(1+B_m(\delta))))\\
    &\textstyle\times \exp(\lambda_{m+1} \mathbbm{1}_{E_{m}(\delta)}C_\tref{lem bound all} T_m^2).
\end{align}
To simplify notation, define 
\begin{align}
    &\textstyle D_{1,m} \doteq \frac{\lambda_{m+1}}{\lambda_m}(1 - T_m \kappa'+4\lambda_{m+1}T_m^2C_{\tref{lem bound z2}}^2(1+B_m(\delta))),\\
    &\textstyle D_{2,m} \doteq \lambda_{m+1} C_\tref{lem bound all} T_m^2.
\end{align}
Then, the previous inequality reads
\begin{align}
  \label{eq exp bound 2}
    &\textstyle\E\qty[\exp(\lambda_{m+1} \mathbbm{1}_{E_{m+1}(\delta)}\norm{w_{t_{m+1}} - w_*}_M^2)| \fF_{t_m}]\\
    \leq&\textstyle \exp(D_{1,m}\lambda_m \mathbbm{1}_{E_{m}(\delta)}\norm{w_{t_m} - w_*}_M^2)\exp(D_{2,m}\mathbbm{1}_{E_{m}(\delta)}).
\end{align}
Next, we bound $D_{1,m}$ and $D_{2,m}$. The input condition gives
\begin{align}
\frac{1+B_m(\delta)}{B_m(\delta)}\leq2.
\end{align}
Using $B_{m+1}(\delta)\geq B_m(\delta)$, $T_m/T_{m+1}=(m+4)/(m+3)\leq2$, and $1-T_m\kappa'\geq0$, we obtain
\begin{align}
D_{1,m}
&=\frac{T_mB_m(\delta)}{T_{m+1}B_{m+1}(\delta)}
\left(1-T_m\kappa'+\frac{4\theta T_m^2C_{\tref{lem bound z2}}^2(1+B_m(\delta))}{T_{m+1}B_{m+1}(\delta)}\right)\\
&\leq\frac{T_m}{T_{m+1}}
\left(1-T_m\kappa'+\frac{4\theta T_m^2C_{\tref{lem bound z2}}^2}{T_{m+1}}\frac{1+B_m(\delta)}{B_m(\delta)}\right)\\
&\leq\frac{T_m}{T_{m+1}}
\left(1-T_m\kappa'+16\theta T_mC_{\tref{lem bound z2}}^2\right).
\end{align}
Recall that $C_\alpha\geq\bar C_\alpha=2/\kappa'>1/\kappa'$.
Select any $\kappa''\in(1/C_\alpha,\kappa')$ and set
\begin{align}
\theta\doteq\min\left\{1,\frac{\kappa'-\kappa''}{16C_{\tref{lem bound z2}}^2}\right\}>0.
\end{align}
Then, for every $m\geq m_0$,
\begin{align}
0\leq D_{1,m}
&\leq\frac{m+4}{m+3}\left(1-\frac{C_\alpha\kappa''}{m+3}\right)\\
&=1-\frac{C_\alpha\kappa''-1}{m+3}-\frac{C_\alpha\kappa''}{(m+3)^2}<1.
\end{align}
Also, since $B_{m+1}(\delta)\geq1$ and $\theta\leq1$,
\begin{align}
D_{2,m}=\frac{\theta C_\tref{lem bound all}T_m^2}{T_{m+1}B_{m+1}(\delta)}
\leq\frac{\theta C_\tref{lem bound all}T_m^2}{T_{m+1}}
\leq2C_\tref{lem bound all}T_m.
\end{align}
Using the upper bounds we obtained for the terms $D_{1,m}$ and $D_{2,m}$ in \eqref{eq exp bound 2}, 
we have for any $m \geq m_0$, 
\begin{align}
&\textstyle\mathbb{E}\qty[\exp(\lambda_{m+1} \mathbbm{1}_{E_{m+1}(\delta)}\norm{w_{t_{m+1}} - w_*}_M^2) | \mathcal{F}_{t_m}] \\
\leq&\textstyle \exp(D_{1,m}\lambda_m \mathbbm{1}_{E_{m}(\delta)}\norm{w_{t_m} - w_*}_M^2)\exp(D_{2,m}\mathbbm{1}_{E_{m}(\delta)})\\
    \leq&\textstyle \exp(\lambda_m \mathbbm{1}_{E_{m}(\delta)}\norm{w_{t_m} - w_*}_M^2)\exp\left(2C_\tref{lem bound all}T_m\right).
\end{align}
For $m \geq m_0$, define $Z_m \doteq \exp(\lambda_m\mathbbm{1}_{E_m(\delta)}\norm{w_{t_m} - w_*}^2_M - 2C_\tref{lem bound all}\sum_{i=m_0}^{m-1}T_i)$. 
The process is adapted to $\{\mathcal{F}_{t_m}\}$ and integrable, since $0<Z_m\leq\exp(\lambda_m B_m(\delta))=\exp(\theta/T_m)$.
We next show that $\qty{Z_m}$ is a supermartingale with respect to the filtration $\{\mathcal{F}_{t_m}\}$.
Indeed, it holds that
\begin{align}
  \textstyle\E\qty[Z_{m+1} | \fF_{t_m}] =&\textstyle \mathbb{E}\qty[\exp(\lambda_{m+1} \mathbbm{1}_{E_{m+1}(\delta)}\norm{w_{t_{m+1}} - w_*}_M^2) | \mathcal{F}_{t_m}] \exp(- 2C_\tref{lem bound all}\sum_{i=m_0}^{m}T_i) \\
  \leq&\textstyle\exp(\lambda_m \mathbbm{1}_{E_{m}(\delta)}\norm{w_{t_m} - w_*}_M^2)\exp\left(2C_\tref{lem bound all}T_m\right)\exp(- 2C_\tref{lem bound all}\sum_{i=m_0}^{m}T_i) \\
  =&\textstyle Z_m.
\end{align}
For any $\delta' \in (0, 1 - \delta)$, 
we now construct $B_m(\delta, \delta')$.
By Ville's maximal inequality,
we get for any $\epsilon$,
\begin{align}
  &\textstyle\Pr(\sup_{m\geq m_0} \exp(\lambda_m \mathbbm{1}_{E_m(\delta)}\norm{w_{t_m} - w_*}^2_M - 2C_\tref{lem bound all} \sum_{i=m_0}^{m-1} T_i) \geq \exp(\epsilon)) \\
  \leq&\textstyle \E\qty[{\exp(\lambda_{m_0} \mathbbm{1}_{E_{m_0}(\delta)}\norm{w_{t_{m_0}} - w_*}^2_M -\epsilon)}] \\ 
  \leq&\textstyle \E\qty[{\exp({\theta}{T^{-1}_{m_0} B_{m_0}(\delta)^{-1}} \mathbbm{1}_{E_{m_0}(\delta)}\norm{w_{t_{m_0}} - w_*}^2_M -\epsilon)}] \\
  \leq&\textstyle \E\qty[{\exp({\theta}T^{-1}_{m_0} -\epsilon)}] = {\exp({\theta}T^{-1}_{m_0} -\epsilon)}.
\end{align}
Select $\epsilon$ such that $\exp(\theta T_{m_0}^{-1} - \epsilon) = \delta'$,
then it holds, with probability at least $1-\delta'$,
that for all $m \geq m_0$,
\begin{align}
  \textstyle\lambda_m \mathbbm{1}_{E_m(\delta)}\norm{w_{t_m} - w_*}^2_M - 2C_\tref{lem bound all} \sum_{i=m_0}^{m-1} T_i \leq \theta T_{m_0}^{-1} - \ln \delta',
\end{align}
implying
\begin{align}
  \textstyle\lambda_m \mathbbm{1}_{E_m(\delta)}\norm{w_{t_m} - w_*}^2_M \leq \theta T_{m_0}^{-1} + \ln (1/\delta') + 2C_\tref{lem bound all}C_\alpha \ln (m+3).
\end{align}
Next, we complete Lemma~\ref{lem inductive step} by removing $\mathbbm{1}_{E_m(\delta)}$ in the LHS of the above inequality. 
For simplicity of notation, 
we denote the RHS of the above inequality as $\epsilon(m,m_0,\delta')$.
We then have
\begin{align}
&\textstyle\Pr(\lambda_m\|w_{t_m}-w_*\|_M^2 \leq \epsilon(m,m_0,\delta'), \forall m \geq m_0) \\
=&\textstyle \Pr\left(\bigcap_{m=m_0}^{\infty}\{\lambda_m\|w_{t_m}-w_*\|_M^2 \leq \epsilon(m,m_0,\delta')\}\right) \\
\geq&\textstyle \Pr\left(\bigcap_{m=m_0}^{\infty}\left(\{\lambda_m\mathbbm{1}_{E_m(\delta)}\|w_{t_m}-w_*\|_M^2 \leq \epsilon(m,m_0,\delta')\} \cap E_m(\delta)\right)\right).
\end{align}
To proceed, note that for any two events, $A$ and $B$, we have
\begin{align}
\Pr(A \cap B) = 1 - \Pr(A^c \cup B^c) \geq 1 - \Pr(A^c) - \Pr(B^c) = \Pr(A) + \Pr(B) - 1.
\end{align}
Therefore, we have
\begin{align}
&\textstyle\Pr( \lambda_m \|w_{t_m}-w_*\|_M^2 \leq \epsilon(m,m_0,\delta'), \forall m \geq m_0) \\
\geq&\textstyle \Pr\left(\bigcap_{m=m_0}^{\infty}\left(\{\lambda_m\mathbbm{1}_{E_m(\delta)}\|w_{t_m}-w_*\|_M^2 \leq \epsilon(m,m_0,\delta')\} \cap E_m(\delta)\right)\right) \\
\geq&\textstyle \Pr\left(\bigcap_{m=m_0}^{\infty}\{\lambda_m\mathbbm{1}_{E_m(\delta)}\|w_{t_m}-w_*\|_M^2 \leq \epsilon(m,m_0,\delta')\}\right) + \Pr\left(\bigcap_{m=m_0}^{\infty}E_m(\delta)\right) - 1 \\
=&\textstyle \Pr(\lambda_m\mathbbm{1}_{E_m(\delta)}\|w_{t_m}-w_*\|_M^2 \leq \epsilon(m,m_0,\delta'), \forall m \geq m_0) + \lim_{m\to\infty}\mathbb{P}(E_m(\delta)) - 1 \\
\geq&\textstyle (1-\delta') + (1-\delta) - 1 \\
=&\textstyle 1-\delta-\delta'.
\end{align}
Using the definitions of $\epsilon(m,m_0,\delta')$ and $\lambda_m$, we arrive at the following result.
For any $\delta' \in (0,1-\delta)$, 
it holds, with probability at least $1-\delta-\delta'$, 
that for any $m \geq m_0$
\begin{align}
\textstyle\|w_{t_m}-w_*\|_M^2 &\leq \frac{T_mB_m(\delta)}{\theta} \qty(\theta T_{m_0}^{-1} + \ln (1/\delta') + 2C_\tref{lem bound all}C_\alpha \ln (m+3)).
\end{align}
It suffices to choose
\begin{align}
C_{\tref{lem inductive step}}\geq
\max\left\{1,\frac{1}{C_\alpha},T_{m_0}^{-1},\frac{1}{\theta},
\frac{2C_\tref{lem bound all}C_\alpha}{\theta}\right\}.
\end{align}
This constant is independent of $m,\delta,\delta'$ and satisfies $C_{\tref{lem inductive step}}C_\alpha\geq1$, which will ensure the lower bound on subsequent input envelopes.
\end{proof}

\subsection{Proof of Lemma~\ref{lem concentration in m}}
\label{proof of lem concentration in m}
\begin{proof}
For any $\delta \in (0, 1)$,
we define
$B_m(\delta) = C_\tref{lem inductive basis}'(m+3)^{C_\tref{lem inductive basis}}$.
From Lemma~\ref{lem inductive basis},
we conclude that
\begin{align}
\Pr(\|w_{t_m}-w_*\|_M^2 \leq B_m(\delta),\ \forall m \geq m_0) = 1 \geq 1 - \delta.
\end{align}
This lays out the antecedent term in the implication relationship specified in Lemma~\ref{lem inductive step}.
We define $C_{\tref{lem concentration in m}} \doteq C_\tref{lem inductive basis} + 1$.
Let $\delta_1,\delta_2,\cdots,\delta_{C_{\tref{lem concentration in m}}} > 0$ be such that $\sum_{i=1}^{C_{\tref{lem concentration in m}}} \delta_i < 1$.
The initial bound holds with probability one, so the first application uses input failure probability zero; each subsequent application adds $\delta_i$ to the accumulated failure probability.
Before application $j+1$, for $0\leq j\leq C_\tref{lem inductive basis}$, the input envelope has the form
\begin{align}
\textstyle C_\tref{lem inductive basis}' C_\alpha^j(m+3)^{C_\tref{lem inductive basis}-j}
\prod_{i=1}^{j}C_{\tref{lem inductive step}}\left[\ln(1/\delta_i)+1+\ln(m+3)\right],
\end{align}
with the empty product equal to one.
It is deterministic, non-decreasing in $m$, and at least one, since $C_\tref{lem inductive basis}'\geq1$, $C_{\tref{lem inductive step}}C_\alpha\geq1$, and $C_\tref{lem inductive basis}-j\geq0$.
Thus every input satisfies the conditions of Lemma~\ref{lem inductive step}; the final output is not used as another input.
Repeatedly applying Lemma~\ref{lem inductive step} for $C_{\tref{lem concentration in m}}$ times,
we have that
it holds, with probability at least $1-\sum_{i=1}^{C_{\tref{lem concentration in m}}} \delta_i$, that for any $m \geq m_0$,
\begin{align}
\textstyle\|w_{t_m}-w_*\|_M^2 \leq {T_m^{C_{\tref{lem concentration in m}}} B_m(\delta)}\prod_{i=1}^{C_{\tref{lem concentration in m}}} C_{\tref{lem inductive step}}\left[\ln(1/\delta_i) + 1 + \ln (m+3) \right].
\end{align}
By choosing $\delta_i = \delta/C_{\tref{lem concentration in m}}$ for all $i \in \{1,2,\cdots,C_{\tref{lem concentration in m}}\}$,
the previous inequality reads
\begin{align}
\textstyle\|w_{t_m}-w_*\|_M^2 &\textstyle\leq {C_{\tref{lem inductive step}}^{C_{\tref{lem concentration in m}}} T_m^{C_{\tref{lem concentration in m}}} B_m(\delta)}\left[\ln(C_{\tref{lem concentration in m}}/\delta) + 1 + \ln (m+3) \right]^{C_{\tref{lem concentration in m}}} \\
&\textstyle= C_{\tref{lem inductive step}}^{C_{\tref{lem concentration in m}}} C_\alpha^{C_{\tref{lem concentration in m}}} C_\tref{lem inductive basis}' (m+3)^{-1} \left[\ln(C_{\tref{lem concentration in m}}/\delta) + 1 + \ln (m+3) \right]^{C_{\tref{lem concentration in m}}}.
\end{align}
Notably, the above inequality holds only for $m \geq m_0$.
For $m < m_0$,
Lemma~\ref{lem inductive basis} gives a trivial almost sure bound that
$\norm{w_{t_m} - w_*}_M^2 \leq C_\tref{lem inductive basis}' (m_0+3)^{C_\tref{lem inductive basis}}$.
To get a bound for all $m$,
we select a constant $C_{\tref{lem concentration in m}, 1}\geq1$ such that
\begin{align}
  \textstyle C_{\tref{lem concentration in m}, 1} \min_{m=0, \dots, m_0}
  C_{\tref{lem inductive step}}^{C_{\tref{lem concentration in m}}} C_\alpha^{C_{\tref{lem concentration in m}}} C_\tref{lem inductive basis}' (m+3)^{-1}
  \left[\ln C_{\tref{lem concentration in m}} + 1 + \ln (m+3) \right]^{C_{\tref{lem concentration in m}}} > 1.
\end{align}
Since $\ln (C_\tref{lem concentration in m} / \delta) \geq \ln C_\tref{lem concentration in m}$,
the same lower bound holds with $\ln C_{\tref{lem concentration in m}}$ replaced by $\ln(C_{\tref{lem concentration in m}}/\delta)$.
It then holds, with probability at least $1-\delta$, that for any $m \geq 0$,
\begin{align}
  \label{eq combining bounds}
  \norm{w_{t_m}-w_*}_M^2
  &\textstyle\leq C_\tref{lem inductive basis}' (m_0+3)^{C_\tref{lem inductive basis}}
  C_{\tref{lem concentration in m}, 1} C_{\tref{lem inductive step}}^{C_{\tref{lem concentration in m}}} C_\alpha^{C_{\tref{lem concentration in m}}} C_\tref{lem inductive basis}'\\
  &\textstyle\qquad\times(m+3)^{-1}\left[\ln(C_{\tref{lem concentration in m}}/\delta) + 1 + \ln (m+3) \right]^{C_{\tref{lem concentration in m}}}.
\end{align}
Defining $C_{\tref{lem concentration in m}}'$ accordingly then completes the proof.
\end{proof}

\subsection{Proof of Theorem~\ref{thm concentration}}
\label{sec map the concentration bound back}
\begin{proof}
The proof is analogous to the proof of Theorem~\ref{thm lil} in Section~\ref{sec proof lem connect skeleton sa and sa}.
We recall that we proceed under Assumption~\ref{assu lr2} and have
$\alpha_t = \frac{C_\alpha}{(t+3)\ln^\omega(t+3)},\ T_m = \frac{C_\alpha}{m+3}$.
Fix $\delta\in(0,1)$ and take $m_0\geq1$ sufficiently large.
Applying the bound in Lemma~\ref{lem concentration in m} to~\eqref{eq connection basic relation},
and using $\norm{w_{t_m}-w_*}^2\leq u_{cm}^2\norm{w_{t_m}-w_*}_M^2$,
we have, with probability at least $1-\delta$, simultaneously for all $m\geq m_0$ and $t\in[t_m,t_{m+1})$,
\begin{align}
    \textstyle\norm{w_t - w_*}^2 \leq 8C_{\tref{lem bound gronwall}}^2\left(T_m^2 + \frac{u_{cm}^2 C_{\tref{lem concentration in m}}'}{m+3} \left[\ln(C_{\tref{lem concentration in m}}/\delta) + 1 + \ln (m+3) \right]^{C_{\tref{lem concentration in m}}}\right).
\end{align}
We work on this event for the remainder of the proof.
Using $\sum_{i=m_0}^{m-1}T_i\leq\sum_{t=t_{m_0}}^{t_m-1}\alpha_t$ and integral comparison (cf. Section~\ref{sec proof lem lr bounds}), we get
\begin{align}
    \textstyle C_\alpha \frac{\ln^{1-\omega}(t_m+2) - \ln^{1-\omega}(t_{m_0}+2)}{1-\omega}
    \geq C_\alpha\bigl(\ln(m+3) - \ln(m_0+3)\bigr).
\end{align}
Hence $m+3 \leq (m_0+3)\exp(\frac{\ln^{1-\omega}(t_m+2)}{1-\omega})$.
Using $\bar\alpha_i\leq T_i+C_{\tref{lem lr bounds}}T_i^2$ for $i\geq m_0$ and integral comparison, we get
\begin{align}
      \textstyle C_\alpha \frac{\ln^{1-\omega}(t_{m+1}+3) - \ln^{1-\omega}(t_{m_0}+3)}{1-\omega}
      \leq C_\alpha\bigl(\ln(m+3) - \ln(m_0+2)\bigr)
      + \sum_{i=m_0}^m \frac{C_{\tref{lem lr bounds}}C_\alpha^2}{(i+3)^2}.
\end{align}
Since the series on the right converges, there exists a positive deterministic constant $C_{{Th}\tref{thm concentration},1}$ such that
$m+3 \geq C_{{Th}\tref{thm concentration},1}\exp(\frac{\ln^{1-\omega}(t_{m+1}+3)}{1-\omega})$.
Then for all $m \geq m_0$ and any $t \in [t_m, t_{m+1})$, we have,
for some positive deterministic constants independent of $m,t,\delta$,
\begin{align}
    &\textstyle\norm{w_t - w_*}^2 \leq 8C_{\tref{lem bound gronwall}}^2\left(\frac{C_\alpha^2}{(m+3)^2} + \frac{u_{cm}^2 C_{\tref{lem concentration in m}}'}{m+3} \left[\ln(C_{\tref{lem concentration in m}}/\delta) + 1 + \ln (m+3) \right]^{C_{\tref{lem concentration in m}}}\right)\\
    \leq &\textstyle C_{{Th}\tref{thm concentration},2}\frac{1}{m+3}\left[\ln(1/\delta) + C_{{Th}\tref{thm concentration},3} + \ln (m+3)\right]^{C_{\tref{lem concentration in m}}} \\
    \leq &\textstyle \frac{C_{{Th}\tref{thm concentration},2}}{C_{{Th}\tref{thm concentration},1}}
    \exp\left(-\frac{\ln^{1-\omega}(t_{m+1}+3)}{1-\omega}\right)
    \left[\ln(1/\delta) + C_{{Th}\tref{thm concentration},3} + \ln(m_0+3) + \frac{\ln^{1-\omega}(t_m+2)}{1-\omega}\right]^{C_{\tref{lem concentration in m}}} \\
    \leq &\textstyle C_{{Th}\tref{thm concentration}}\exp\left(-\frac{\ln^{1-\omega}(t+1)}{1-\omega}\right)
    \left[\ln(1/\delta) + C_{{Th}\tref{thm concentration}}' + \frac{\ln^{1-\omega}(t+1)}{1-\omega}\right]^{C_{\tref{lem concentration in m}}}.
\end{align}
In the last step, we use $t_m\leq t<t_{m+1}$ and
$0\leq\ln^{1-\omega}(t+2)-\ln^{1-\omega}(t+1)\leq\ln^{1-\omega}2$,
absorbing the resulting constant into $C_{{Th}\tref{thm concentration}}'$.
For all $t<t_{m_0}$, we can use a deterministic bound from~\eqref{eq ttm gronwall 3}.
Taking $C_{{Th}\tref{thm concentration}}'\geq1$ and increasing $C_{{Th}\tref{thm concentration}}$ covers these finitely many iterates as well.
Combining the two bounds (cf.~\eqref{eq combining bounds}) then completes the proof.
\end{proof}

\subsection{Proof of Corollary~\ref{cor lr2 as}}
\label{sec proof cor lr2 as}
\begin{proof}
For each integer $n\geq2$, let $E_n$ be the event that
\begin{align}
\norm{w_t-w_*}^2
\leq C_{\mathrm{Th}\tref{thm concentration}}e^{-s(t)}
\bigl[\log n+C'_{\mathrm{Th}\tref{thm concentration}}+s(t)\bigr]^{C_{\tref{lem concentration in m}}}
\end{align}
holds for every $t\geq0$. Applying Theorem~\ref{thm concentration} with
$\delta=1/n$ gives $\Pr(E_n)\geq1-1/n$. Therefore, for every integer $n\geq2$,
\begin{align}
1\geq\Pr\!\left(\bigcup_{j\geq2}E_j\right)
\geq\Pr(E_n)\geq1-\frac1n.
\end{align}
Letting $n\to\infty$ shows that $\Pr(\bigcup_{j\geq2}E_j)=1$.

Fix a sample path in $\bigcup_{j\geq2}E_j$. There is a finite integer
$n\geq2$ such that the path belongs to $E_n$.
Since this path belongs to $E_n$, the inequality in the definition of $E_n$
holds along this path for every $t\geq0$.
Since $s(t)\geq0$ and $s(t)\to\infty$, for all sufficiently large $t$,
\begin{align}
0\leq\log n+C'_{\mathrm{Th}\tref{thm concentration}}+s(t)
\leq\bigl(1+\log n+|C'_{\mathrm{Th}\tref{thm concentration}}|\bigr)(1+s(t)).
\end{align}
Consequently, along this path and for all sufficiently large $t$,
\begin{align*}
\norm{w_t-w_*}^2
&\leq C_{\mathrm{Th}\tref{thm concentration}}
\bigl(1+\log n+|C'_{\mathrm{Th}\tref{thm concentration}}|\bigr)^{C_{\tref{lem concentration in m}}}
 e^{-s(t)}(1+s(t))^{C_{\tref{lem concentration in m}}}.
\end{align*}
The factor preceding $e^{-s(t)}(1+s(t))^{C_{\tref{lem concentration in m}}}$
is finite and independent of $t$; it may depend on the sample path through $n$.
This proves the first assertion on an event of probability one.

For the second assertion, fix any $c\in(0,1)$. Multiplying the preceding
bound by $e^{c s(t)}$, its time-dependent factor becomes
$e^{-(1-c)s(t)}(1+s(t))^{C_{\tref{lem concentration in m}}}$. Its logarithm satisfies
\begin{align}
-(1-c)s(t)+C_{\tref{lem concentration in m}}\log(1+s(t))\longrightarrow-\infty,
\end{align}
because $s(t)\to\infty$ and $\log(1+s(t))/s(t)\to0$.
Thus $e^{c s(t)}\norm{w_t-w_*}^2\to0$ along the fixed path.
The event $\bigcup_{j\geq2}E_j$ does not depend on $c$, so this conclusion
holds on that same probability-one event for every $c\in(0,1)$.
\end{proof}

\subsection{Proof of Corollary~\ref{cor lp}}
\label{sec proof cor lp}
\begin{proof}
Fix $t\geq0$ and a real $p\geq1$. Recall $a(t)>0$, $b(t)\geq0$,
and $C_{\tref{lem concentration in m}}\geq1$ from Corollary~\ref{cor lp}, so $C_{\tref{lem concentration in m}}p\geq1$. By Theorem~\ref{thm concentration},
\begin{align}
\Pr\{\norm{w_t-w_*}^2>a(t)(b(t)+u)^{C_{\tref{lem concentration in m}}}\}\leq e^{-u},\qquad u>0.
\end{align}
For $x>a(t)b(t)^{C_{\tref{lem concentration in m}}}$, setting $u=(x/a(t))^{1/C_{\tref{lem concentration in m}}}-b(t)>0$ therefore gives
\begin{align}
\Pr\{\norm{w_t-w_*}^2>x\}\leq\exp\bigl(b(t)-(x/a(t))^{1/C_{\tref{lem concentration in m}}}\bigr).
\end{align}
For $0\leq x\leq a(t)b(t)^{C_{\tref{lem concentration in m}}}$, we use the bound $\Pr\{\norm{w_t-w_*}^2>x\}\leq1$.
Since $\norm{w_t-w_*}^2\geq0$, Tonelli's theorem applied to
\begin{align}
\norm{w_t-w_*}^{2p}=\int_0^\infty p x^{p-1}\mathbbm{1}_{\{\norm{w_t-w_*}^2>x\}}\dd x
\end{align}
yields
\begin{align*}
\E[\norm{w_t-w_*}^{2p}]
&=\int_0^\infty p x^{p-1}\Pr\{\norm{w_t-w_*}^2>x\}\dd x\\
&=\underbrace{\int_0^{a(t)b(t)^{C_{\tref{lem concentration in m}}}}p x^{p-1}\Pr\{\norm{w_t-w_*}^2>x\}\dd x}_{I_1}\\
&\quad+\underbrace{\int_{a(t)b(t)^{C_{\tref{lem concentration in m}}}}^\infty p x^{p-1}\Pr\{\norm{w_t-w_*}^2>x\}\dd x}_{I_2}.
\end{align*}
This identity holds even before finiteness of the moment has been established.
The first integral satisfies
\begin{align}
I_1\leq\int_0^{a(t)b(t)^{C_{\tref{lem concentration in m}}}}p x^{p-1}\dd x=a(t)^p b(t)^{C_{\tref{lem concentration in m}}p}.
\end{align}
For the second, the change of variable $u=(x/a(t))^{1/C_{\tref{lem concentration in m}}}$ gives
\begin{align*}
I_2
&\leq\int_{a(t)b(t)^{C_{\tref{lem concentration in m}}}}^\infty p x^{p-1}
       \exp\bigl(b(t)-(x/a(t))^{1/C_{\tref{lem concentration in m}}}\bigr)\dd x\\
&=C_{\tref{lem concentration in m}}p a(t)^p e^{b(t)}\int_{b(t)}^\infty u^{C_{\tref{lem concentration in m}}p-1}e^{-u}\dd u.
\end{align*}
To retain the dependence on the lower integration limit $b(t)$, integrate by parts:
\begin{align*}
C_{\tref{lem concentration in m}}p\int_{b(t)}^\infty u^{C_{\tref{lem concentration in m}}p-1}e^{-u}\dd u
&=\int_{b(t)}^\infty e^{-u}\,\dd\!\left(u^{C_{\tref{lem concentration in m}}p}\right)\\
&=\left[u^{C_{\tref{lem concentration in m}}p}e^{-u}\right]_{b(t)}^\infty
  -\int_{b(t)}^\infty u^{C_{\tref{lem concentration in m}}p}\,\dd(e^{-u})\\
&=-b(t)^{C_{\tref{lem concentration in m}}p} e^{-b(t)}+\int_{b(t)}^\infty u^{C_{\tref{lem concentration in m}}p} e^{-u}\dd u.
\end{align*}
Combining the bounds for $I_1$ and $I_2$, substituting the integration-by-parts identity above, and then setting $z=u-b(t)$, yields
\begin{align*}
\E[\norm{w_t-w_*}^{2p}]
&\leq a(t)^p b(t)^{C_{\tref{lem concentration in m}}p}+C_{\tref{lem concentration in m}}p a(t)^p e^{b(t)}\int_{b(t)}^\infty u^{C_{\tref{lem concentration in m}}p-1}e^{-u}\dd u\\
&=a(t)^p b(t)^{C_{\tref{lem concentration in m}}p}
  +a(t)^p e^{b(t)}\left[-b(t)^{C_{\tref{lem concentration in m}}p}e^{-b(t)}
  +\int_{b(t)}^\infty u^{C_{\tref{lem concentration in m}}p}e^{-u}\dd u\right]\\
&=a(t)^p b(t)^{C_{\tref{lem concentration in m}}p}-a(t)^p b(t)^{C_{\tref{lem concentration in m}}p}
  +a(t)^p e^{b(t)}\int_{b(t)}^\infty u^{C_{\tref{lem concentration in m}}p}e^{-u}\dd u\\
&=a(t)^p e^{b(t)}\int_{b(t)}^\infty u^{C_{\tref{lem concentration in m}}p} e^{-u}\dd u\\
&=a(t)^p\int_0^\infty(b(t)+z)^{C_{\tref{lem concentration in m}}p} e^{-z}\dd z.
\end{align*}
Finally, $(b(t)+z)^{C_{\tref{lem concentration in m}}p}\leq2^{C_{\tref{lem concentration in m}}p-1}(b(t)^{C_{\tref{lem concentration in m}}p}+z^{C_{\tref{lem concentration in m}}p})$ for $C_{\tref{lem concentration in m}}p\geq1$ and $b(t),z\geq0$, so
\begin{align*}
\E[\norm{w_t-w_*}^{2p}]
&\leq2^{C_{\tref{lem concentration in m}}p-1}a(t)^p\left[b(t)^{C_{\tref{lem concentration in m}}p}\int_0^\infty e^{-z}\dd z
                +\int_0^\infty z^{C_{\tref{lem concentration in m}}p} e^{-z}\dd z\right]\\
&=2^{C_{\tref{lem concentration in m}}p-1}a(t)^p\bigl[b(t)^{C_{\tref{lem concentration in m}}p}+\Gamma(C_{\tref{lem concentration in m}}p+1)\bigr].
\end{align*}
The Gamma integral is finite for every real $p\geq1$.
Since $a(t)=C_{\mathrm{Th}\tref{thm concentration}}e^{-s(t)}$ and
$b(t)=O(1+s(t))$, the stated convergence rate follows, including $p=1$.
\end{proof}

\subsection{Proof of Proposition~\ref{prop lr2 lower}}
\label{sec proof prop lr2 lower}
\begin{proof}
Fix any $\omega\in(0,1)$. Consider the scalar, noise-free instance
$H(w,y)=h(w)=w/2$, with $w_0=1$ and a one-state Markov chain.
All our operator and mixing assumptions hold, with $\kappa=1/2$ and $w_*=0$.
Here both norms are the absolute value, so $\norm{w_t-w_*}^2=|w_t-w_*|^2$.
Choose $C_\alpha=5$. The sufficient concentration threshold in~\eqref{eq concentration threshold} is $4$, so this choice satisfies the conditions of Theorem~\ref{thm concentration}.
The recursion is exactly
\begin{align}
 w_t=\prod_{k=0}^{t-1}(1-b_k),\qquad
 b_k=\frac{5}{2(k+3)\ln^\omega(k+3)}.
\end{align}
Here $0<b_k\leq b_0<5/6<1$ and $\sum_k b_k^2<\infty$.
Consequently, $\log(1-b_k)=-b_k+O(b_k^2)$, uniformly in $k$,
and integral comparison gives
\begin{align}
 \sum_{k=0}^{t-1}b_k
 &=\frac{5}{2(1-\omega)}\ln^{1-\omega}(t+3)+O(1),\\
 \log |w_t|^2
 &=-\frac{5}{1-\omega}\ln^{1-\omega}(t+3)+O(1).
\end{align}
Thus $|w_t|^2=\Theta\bigl(\exp\{-5\ln^{1-\omega}(t+3)/(1-\omega)\}\bigr)$,
and $t^r|w_t|^2\to\infty$ for every $r>0$.
\end{proof}




\clearpage
\section{Experimental Details}
\label{sec experimental details}
This appendix describes the experimental setup and the computation of the reference solution.

\tb{MDP and behavior policy.}
The experiment in Section~\ref{sec lr2 experiment} uses a continuing MDP with three states $\{0,1,2\}$, two actions $\{0,1\}$,
and discount factor $\gamma=0.5$. The transition probabilities are
\begin{align}
 \textstyle p(s'|s,0)&=0.8\mathbbm{1}_{s'=s}+0.2/3, &
 p(s'|s,1)&=0.8\mathbbm{1}_{s'=(s+1)\bmod 3}+0.2/3,
\end{align}
and the reward matrix, with states indexing rows and actions indexing columns, is
$r=\left(\begin{smallmatrix}0&0.2\\1&-0.1\\0.3&0.7\end{smallmatrix}\right)$.
The fixed behavior policy chooses each action with probability $1/2$.
All transition probabilities are positive, and the stationary state-action
distribution is uniform. Thus the sampling chain is irreducible and aperiodic,
and the averaged asynchronous operator has contraction modulus $\kappa=11/12$.

\tb{Stepsize and initialization.}
Each run starts from $S_0=0$ and $q_0=0$, follows a single Markovian trajectory,
and updates only the visited state-action coordinate. We use the exact schedule
$\alpha_t=32/((t+3)\ln^\omega(t+3))$ for $\omega\in\{0.2,0.5,0.8\}$,
without clipping or projection. The coefficient $32$ satisfies the sufficient
threshold in the current analysis: choosing the Euclidean smooth norm and
$\xi=0.01$ gives $l_{cs}=1/\sqrt{6}$, $u_{cs}=1$, and
$\kappa'=1-(11/12)\sqrt{1.01/(1+0.01/6)}\approx0.07953$,
so $2/\kappa'\approx25.15<32$.
Initial stepsizes exceed one, and the resulting transient is included.

\tb{Runs and reference solution.}
We run $10^6$ updates for each of 128 independent random seeds (0--127),
sharing each seed's trajectory across the three exponents.
The reference $q_*$ is computed by value iteration and checked independently
by enumerating all eight deterministic policies, with Bellman residual below
$10^{-12}$.

\vskip 0.2in
\bibliography{bibliography.bib}

\end{document}